\documentclass{article}

\usepackage{microtype}
\usepackage{graphicx}
\usepackage{subfigure}
\usepackage{booktabs} 

\usepackage{hyperref}

\usepackage{color}
\usepackage{amsfonts}       
\usepackage{url}
\usepackage{amsmath}
\usepackage{amssymb}
\usepackage{wrapfig}
\usepackage{mathtools}
\usepackage{tikz}
\usepackage{footmisc}
\usepackage{bibentry}
\usepackage{thmtools, thm-restate}
\nobibliography*
\usepackage{mathtools}

\usepackage{physics}
\usepackage{mathrsfs}
\mathtoolsset{showonlyrefs}
\usepackage{etoolbox}
\usepackage{makecell}
\usepackage{enumerate}

\usepackage[accepted]{icml2023}

\usepackage{amsmath}
\usepackage{amssymb}
\usepackage{mathtools}
\usepackage{amsthm}
\usepackage{pifont}

\newcommand{\cmark}{\ding{51}}%
\newcommand{\xmark}{\ding{55}}%

\newcommand{\fS}{\mathcal{S}}
\newcommand{\fA}{\mathcal{A}}
\newcommand{\fY}{\mathcal{Y}}

\newcommand{\fO}{\mathcal{O}}
\newcommand{\fH}{\mathcal{H}}

\newcommand{\R}{\mathbb{R}}
\newcommand{\E}{\mathbb{E}}
\newcommand{\nsa}{{|\fS \times \fA|}}
\newcommand{\ns}{{|\fS|}}
\newcommand{\na}{{|\fA|}}
\newcommand{\ny}{{|\fY|}}
\newcommand{\bop}{\mathcal{T}}

\newcommand{\indot}[2]{{\left<#1, #2\right>}}

\theoremstyle{plain}
\newtheorem{theorem}{Theorem}[section]

\newtheorem{lemma}[theorem]{Lemma}

\theoremstyle{definition}

\newtheorem{assumption}[theorem]{Assumption}
\theoremstyle{remark}

\usepackage[textsize=tiny]{todonotes}

\icmltitlerunning{On the Convergence of SARSA with Linear Function Approximation}

\begin{document}

\twocolumn[
\icmltitle{On the Convergence of SARSA with Linear Function Approximation}



\icmlsetsymbol{equal}{*}

\begin{icmlauthorlist}
\icmlauthor{Shangtong Zhang}{uva}
\icmlauthor{Remi Tachet des Combes}{alp}
\icmlauthor{Romain Laroche}{un}
\end{icmlauthorlist}

\icmlaffiliation{uva}{Department of Computer Science, University of Virginia, United States}
\icmlaffiliation{alp}{AlpacaML}
\icmlaffiliation{un}{Unemployed}

\icmlcorrespondingauthor{Shangtong Zhang}{shangtong@virginia.edu}

\icmlkeywords{Reinforcement Learning, SARSA}

\vskip 0.3in
]



\printAffiliationsAndNotice{}  

\begin{abstract}
    SARSA, a classical on-policy control algorithm for reinforcement learning,
    is known to \emph{chatter} when combined with linear function approximation:
    SARSA does not diverge but oscillates in a bounded region. 
    However,
    little is known about how fast SARSA converges to that region and how large the region is.
    In this paper,
    we make progress towards this open problem by showing the convergence rate of projected SARSA to a bounded region.
    Importantly,
    the region
    is much smaller than the region that we project into, 
    provided that the magnitude of the reward is not too large.
    Existing works regarding the convergence of linear SARSA to a fixed point all require the Lipschitz constant of SARSA's policy improvement operator to be sufficiently small;
    our analysis instead applies to arbitrary Lipschitz constants 
    and thus characterizes the behavior of linear SARSA for a new regime.  
\end{abstract}

\section{Introduction}
\label{sec intro}
SARSA is a classical on-policy control algorithm 
 for reinforcement learning (RL, \citet{sutton2018reinforcement}) 
dating back to \citet{rummery1994line}.
The key idea of SARSA is to update the estimate for action  values with data generated by following an exploratory and greedy policy
(e.g., an $\epsilon$-greedy policy) 
derived from the estimate itself. 
In this paper,
we refer to the operator used for deriving such a policy from the action value estimate as the \emph{policy improvement operator}.

The study of SARSA begins in the tabular setting,
where the action value estimates are stored
in the form of a look-up table.
For example, \citet{singh2000convergence} confirm
the asymptotic convergence of SARSA to the optimal policy
provided that the policies from the policy improvement operator satisfy the ``greedy in the limit with infinite exploration'' condition.
Tabular methods,
however,
are not preferred 
when the state space is large and generalization is required across states.
One possible solution is linear function approximation,
which approximates the action values via the inner product of state-action features and a learnable weight vector. 
The behavior of SARSA with linear function approximation (linear SARSA) is,
however,
less understood.

\begin{table*}[t]
  \footnotesize
  \centering
  \begin{tabular}{c|c|c|c|c|c}
    &\citeauthor{gordon2001reinforcement} & \citeauthor{perkins2003convergent} & \citeauthor{melo2008analysis} & \citeauthor{zou2019finite} & Ours \\ \hline
    convergence & to a region & to a point & to a point & to a point & to a region \\ \hline
    \makecell{per-step \\ policy \\ improvement} & \xmark & \xmark & \cmark & \cmark & \cmark \\ \hline
    \makecell{any \\ Lipschitz \\ constant}& \cmark & \cmark$^*$ & \xmark & \xmark & \cmark \\ \hline
    \makecell{convergence \\ rate} & \xmark & \xmark & \xmark &\cmark & \cmark \\ \hline
  \end{tabular}
  \caption{\label{tab related work} Comparison with existing works. 
  \cmark$^*$ indicates that the corresponding property is not explicitly documented in the original work.
  ``Per-step policy improvement'' means that the policy improvement operator is applied every time step.}
\end{table*}

\citet{gordon1996chattering} and \citet{bertsekas1996neuro} empirically observe that linear SARSA can chatter:
the weight vector does not go to infinity 
(i.e., it does not diverge)
but oscillates in a bounded region.
Importantly,
this chattering behavior remains even if a decaying learning rate is used.
\citet{gordon2001reinforcement} further proves that \emph{trajectory-based} linear SARSA with an $\epsilon$-greedy policy improvement operator converges to a bounded region \emph{asymptotically}.
Unlike standard linear SARSA,
where the policy improvement operator is invoked \emph{every step} to generate a new policy for action selection in the next step,
trajectory-based linear SARSA generates a policy at the beginning of each episode and the policy remains fixed during the episode.
Intuitively,
within an episode,
trajectory-based linear SARSA is just linear Temporal Difference (TD, \citet{sutton1988learning}) learning for evaluating action values.
It,
therefore,
converges to an approximation of the action value function of the policy generated at the beginning of the episode \citep{tsitsiklis1997analysis}.
Since
the number of all possible $\epsilon$-greedy policies is finite in a finite Markov Decision Process with a fixed $\epsilon$,
trajectory-based linear SARSA oscillates among the (approximate) action value functions of those $\epsilon$-greedy policies,
which form a bounded region.
Later, 
\citet{perkins2003convergent} prove the asymptotic convergence of \emph{fitted} linear SARSA
(a.k.a.\ model-free approximate policy iteration) to a fixed point.
Similar to trajectory-based SARSA,
fitted SARSA alternates between \emph{thorough} TD learning for policy evaluation under a fixed policy and the application of the policy improvement operator.
In other words,
it involves bi-level optimization.
Then assuming the Lipschitz constant of the policy improvement operator is sufficiently small such that the composition of the policy improvement operator and some other function becomes contractive,
convergence of fitted SARSA is obtained thanks to Banach's fixed point theorem. 
Despite this progress,
the asymptotic behavior of standard linear SARSA,
which invokes the policy improvement operator every time step,
still remains unclear,
as does a potential convergence rate.
Understanding the behavior of linear SARSA is one of the four open theoretical questions in RL raised by \citet{sutton1999open}.

Several efforts have been made to analyze linear SARSA.
\citet{melo2008analysis} prove the asymptotic convergence of linear SARSA.
\citet{zou2019finite} further provide a convergence rate of a projected linear SARSA,
which uses an additional projection operator in the canonical linear SARSA update.
Unlike \citet{gordon2001reinforcement},
the convergence in \citet{melo2008analysis,zou2019finite} is to a fixed point instead of a bounded region.
Although convergence to a fixed point is preferred,
\citet{melo2008analysis,zou2019finite} require that SARSA's policy improvement operator is Lipschitz continuous and the Lipschitz constant is sufficiently small.
It remains an open problem \emph{how linear SARSA behaves when the Lipschitz constant is large}.

This problem is of interest because to our knowledge,
most meaningful empirical results using SARSA for control
consider an $\epsilon$-greedy policy or a softmax policy.
The former behaves similarly to a Lipschitz continuous policy improvement operator with a very large Lipschitz constant.
The latter is usually tuned such that its Lipschitz constant is reasonably large.
We refer the reader to Section~\ref{sec bkg} for more discussion about those two classes of policy improvement operators 
and now name a few notable empirical results.
In the tabular setting,
\citet{sutton2018reinforcement} use an $\epsilon$-greedy policy with $\epsilon = 0.1$ in the Windy GridWorld.
With linear function approximation,
\citet{rummery1994line} use a softmax policy in a robot control problem with the temperature decaying from $0.05$ to $0.01$ such that the softmax policy has a reasonably large Lipschitz constant.
\citet{liang2015state} use an $\epsilon$-greedy 
policy with $\epsilon=0.01$ to play all Atari games \citep{bellemare13arcade}.
In deep RL,
\citet{mnih2016asynchronous} use $\epsilon$-greedy policies with decaying $\epsilon$ in their asynchronous methods for playing Atari games.
Besides the aforementioned interest from the empirical side,
this problem has also been recognized as an important theoretical open problem in \citet{perkins2003convergent,zou2019finite}.

In this paper,
we make contributions to this open problem.
In particular,
we study the projected linear SARSA \citep{zou2019finite}
and show that it converges to a bounded region regardless of the magnitude of the Lipschitz constant of the policy improvement operator.
Importantly, 
the bounded region is much smaller than the region we project into provided that the magnitude of rewards is not too large.
The differences between our work and existing works are summarized in Table~\ref{tab related work}.

\section{Background}
\label{sec bkg}
In this paper, all vectors are column.
We use $\indot{x}{y} \doteq x^\top y$ to denote the standard inner product in Euclidean spaces.
For a positive definite matrix $D$,
we use $\norm{x}_D \doteq \sqrt{x^\top Dx}$ to denote the vector norm induced by $D$.
We overload $\norm{\cdot}_D$ to also denote the induced matrix norm.
We write $\norm{\cdot}$ as shorthand for $\norm{\cdot}_I$,
where $I$ is the identity matrix,
i.e.,
$\norm{\cdot}$ denotes the standard $\ell_2$ norm.
When it does not cause confusion,
we use vectors and functions interchangeably.
For example, 
if $f$ is a function $\fS \to \R$,
we also use $f$ to denote the vector in $\R^{\ns}$ whose $s$-indexed element is $f(s)$.

We consider an infinite horizon Markov Decision Process (MDP, \citet{puterman2014markov}) 
with a finite state space $\fS$,
a finite action space $\fA$,
a transition kernel $p: \fS \times \fS \times \fA \to [0, 1]$,
a reward function $r: \fS \times \fA \to [-r_{max}, r_{max}]$,
a discount factor $\gamma$,
and an initial distribution $p_0: \fS \to [0, 1]$.
At time step $t=0$,
an initial state $S_0$ is sampled from $p_0(\cdot)$.
At time step $t$,
an agent at a state $S_t$ takes an action $A_t \sim \pi(\cdot | S_t)$ according to a policy $\pi: \fA \times \fS \to [0 ,1]$.
The agent then receives a reward $R_{t+1} \doteq r(S_t, A_t)$ and proceeds to a successor state $S_{t+1} \sim p(\cdot | S_t, A_t)$.
The return at time step $t$ is defined as 
  $G_t \doteq \sum_{i=0}^\infty \gamma^i R_{t+i+1}$,
which allows us to define the action value function as 
\begin{align}
  \textstyle q_\pi(s, a) \doteq \E\left[G_t | S_t = s, A_t = a, \pi, p\right].
\end{align}
The action value function $q_\pi$ is closely related to the Bellman operator $\bop_\pi$,
which is defined as
  $\bop_\pi q \doteq r + \gamma P_\pi q$,
where $P_\pi \in \R^{\nsa \times \nsa}$ is the state-action pair transition matrix,
i.e., $P_\pi\left((s, a), (s', a')\right) \doteq p(a'|s, a) \pi(a'|s')$.
In particular, $q_\pi$ is the only vector $q \in \R^\nsa$ satisfying
  $q = \bop_\pi q$.
The goal of control is to find an optimal policy $\pi_*$ such that $\forall \pi, s, a$,
  $q_{\pi_*}(s, a) \geq q_\pi(s, a)$.
All optimal policies share the same action value function,
which is referred to as $q_*$.
One classical approach for finding $q_*$ is SARSA,
which updates an estimate $q \in \R^\nsa$ iteratively as
\begin{align}
  \label{eq tabular sarsa}
  A_{t+1} &\sim \pi_{q}(\cdot | S_{t+1}), \\
  \delta_t &\gets R_{t+1} + \gamma q(S_{t+1}, A_{t+1}) - q(S_t, A_t) \\
  q(S_t, A_t) &\gets q(S_t, A_t) + \alpha_t \delta_t,
\end{align}
where $\qty{\alpha_t}$ is a sequence of learning rates
and $\pi_q$ denotes that the policy $\pi$ is parameterized by the action value estimate $q$.
A commonly used $\pi_q$ is an $\epsilon$-greedy policy,
i.e.,
\begin{align}
  \label{eq epsilon greedy policy}
  \pi_q(a|s) = \begin{cases}
    \frac{\epsilon}{\abs{\fA}}, \quad a \neq \arg\max_{b \in \fA} q(s, b) \\
    1 - \epsilon + \frac{\epsilon}{\abs{\fA}},  \qq{otherwise}
  \end{cases},
\end{align}
where $\epsilon \in [0, 1]$ is a hyperparameter.
Another common example is an $\epsilon$-softmax policy,
i.e.,
\begin{align}
  \label{eq epsilon softmax policy}
  \textstyle \pi_q(a|s) = \frac{\epsilon}{\na} + (1-\epsilon)\frac{\exp(q(s, a) / \iota)}{\sum_b \exp(q(s, b) / \iota)},
\end{align}
where $\iota \in (0, \infty)$ is the temperature of the softmax function.
This $\pi_q$ is exactly the policy improvement operator discussed in Section~\ref{sec intro}:
it maps an action value estimate $q$ to a new policy;
it is ``improvement'' in that it usually has
greedification over the action value estimate to some extent.
The $\epsilon$-softmax policy $\pi_q$ in~\eqref{eq epsilon softmax policy} is Lipschitz continuous in $q$ with the Lipschitz constant being $\frac{1-\epsilon}{\iota}$.
When the temperature $\iota$ approaches $0$,
the $\epsilon$-softmax policy approaches the $\epsilon$-greedy policy.
Therefore,
despite the $\epsilon$-greedy policy is not even continuous,
it is the limit of a sequence of Lipschitz continuous policies with the Lipschitz constants approaching infinity.
We, therefore,
argue that an $\epsilon$-greedy policy would empiricially behave similarly to a Lipschitz continuous policy with a very large Lipschitz constant.

So far we have considered only time-homogeneous policies.
One can also consider time-inhomogeneous policies, 
e.g., a policy $\pi_{q, t}(a|s)$ that depends on both the action  value estimate $q$ and the time step $t$.
\citet{singh2000convergence} show that if the time-inhomogeneous policies $\pi_{q, t}$ satisfy the ``greedy in the limit with infinite exploration'' (GLIE) condition 
then the iterates generated by~\eqref{eq tabular sarsa} converge to $q_*$ almost surely.

It is,
however,
not always practical to use a look-up table for storing our action value estimates,
especially when the state space is large or generalization is required across states. 
One natural solution is linear function approximation,
where the action value estimate $q(s,a)$ is parameterized as $x(s, a)^\top w$.
Here $x: \fS \times \fA \to \R^K$ is the feature function which maps a state-action pair to a $K$-dimensional vector and $w \in \R^K$ is the learnable weight vector.
We use $X \in \R^{\nsa \times K}$ to denote the feature matrix,
whose $(s, a)$-indexed row is $x(s, a)^\top$.
We use as shorthand
\begin{align}
  \textstyle \pi_w \doteq \pi_{Xw}, \, x_t \doteq x(S_t, A_t), \, x_{max} \doteq \max_{s, a} \norm{x(s, a)}.
\end{align}

\begin{algorithm}[t]
\begin{algorithmic}
  \STATE Initialize $w_0$ such that $\norm{w_0} \leq C_{\Gamma}$
  \STATE $S_0 \sim p_0(\cdot), A_0 \sim \pi_{w_0}(\cdot | S_0)$
  \STATE $t \gets 0$
  \WHILE{true}
\STATE Execute $A_t$, get $R_{t+1}, S_{t+1}$ \;
 \STATE Sample $A_{t+1} \sim \pi_{w_t}(\cdot | S_{t+1})$ \;
 \STATE $\delta_t \gets R_{t+1} + \gamma x_{t+1}^\top w_t - x_t^\top w_t$ \;
 \STATE $w_{t+1} \gets \Gamma\left(w_t + \alpha_t \delta_t x_t\right)$ \;
 \STATE $t \gets t + 1$ \;
  \ENDWHILE
\end{algorithmic}
  \caption{\label{alg sarsa lambda}SARSA with linear function approximation}
\end{algorithm}

SARSA (Algorithm~\ref{alg sarsa lambda}) is a commonly used algorithm for learning $w$. 
In Algorithm~\ref{alg sarsa lambda}, $\Gamma: \R^K \to \R^K$ is a projection operator onto a ball of radius $C_\Gamma$, i.e., 
\begin{align}
  \label{eq projection}
  \textstyle \Gamma(w) \doteq \begin{cases}
    w, & \norm{w} \leq C_\Gamma, \\
    C_\Gamma \frac{w}{\norm{w}}, & \norm{w} > C_\Gamma
  \end{cases}.
\end{align}
For now we consider the setting where $C_\Gamma = \infty$,
i.e.,
$\Gamma$ is an identity mapping.
If the iterates $\qty{w_t}$ generated by SARSA converged to some vector $w_*$,
the expected update at $w_*$ would have to diminish,
i.e.,
\begin{align}
  \label{eq sarsa fixed point expectation}
  \E_{S_t, A_t \sim d_{\pi_{w_*}}}\left[\left(R_{t+1} + \gamma x_{t+1}^\top w_* - x_t^\top w_*\right) x_t\right] = 0,
\end{align}
where for a policy $\pi$, 
we use $d_\pi \in \R^\nsa$ to denote the stationary state action pair distribution of the chain in $\fS \times \fA$ induced by $\pi$ (assuming it exists).
We can equivalently write~\eqref{eq sarsa fixed point expectation} in a matrix form as
\begin{align}
  \label{eq sarsa fixed point expectation 2}
  X^\top D_{\pi_{w_*}}(r + \gamma P_{\pi_{w_*}} Xw_* - Xw_*) = 0,
\end{align}
where for a policy $\pi$,
we use $D_\pi \in \R^{\nsa \times \nsa}$ to denote a diagonal matrix whose diagonal entry is $d_\pi$.
By defining
\begin{align}
  \label{eq source of error function}
  A_{\pi} \doteq X^\top D_\pi(\gamma P_\pi  - I)X, \, b_\pi \doteq X^\top D_\pi r,
\end{align}
\eqref{eq sarsa fixed point expectation 2} becomes
  $A_{\pi_{w_*}} w_* + b_{\pi_{w_*}} = 0$.
It is known 
(see, e.g., \citet{tsitsiklis1997analysis}) that $A_{\pi}$ is negative definite under mild conditions.
We define a projection operator $\Pi_{D_\pi}$ mapping a vector in $\R^\nsa$ to the column space of $X$ as
\begin{align}
  \textstyle \Pi_{D_\pi} q \doteq \arg\min_{\hat q \in col(X)} \norm{\hat q - q}_{D_\pi}^2,
\end{align}
where $col(X)$ denotes the column space of $X$.
It can be computed that 
  $\Pi_{D_\pi} = X \left(X^\top D_\pi X\right)^{-1} X^\top D_\pi$
and it is known (see, e.g., \citet{de2000existence}) that~\eqref{eq sarsa fixed point expectation 2} holds if and only if 
  $\Pi_{D_{\pi_{w_*}}} \bop_{\pi_{w_*}} Xw_* = Xw_*$.
In other words,
$Xw_*$ is a fixed point of the operator
  $\fH(q) \doteq \Pi_{D_{\pi_q}} \bop_{\pi_q} q$.
The operator $\fH$ is referred to as the \emph{approximate policy iteration} operator
and SARSA is an incremental, stochastic method to find a fixed point of approximate policy iteration.
Unfortunately,
when
$\pi_q(a|s)$ is not continuous in $q$
(e.g., $\pi_q$ is an $\epsilon$-greedy policy, c.f.~\eqref{eq epsilon greedy policy}),
$\fH$ does not necessarily have a fixed point \citep{de2000existence}.
Conversely, when $\pi_q(a|s)$ is continuous in $q$,
\citet{de2000existence} show that $\fH$ has at least one fixed point.

\citet{perkins2003convergent} assume $\pi_{q}$ is Lipschitz continuous in $q$ and study a form of fitted SARSA,
which is a model-free variant of approximate policy iteration.
At the $k$-th iteration,
\citet{perkins2003convergent} first invoke TD for learning the action value function of $\pi_k$,
which converges to $w_k = A^{-1}_{\pi_k} b_{\pi_k}$ after infinitely many steps.
Then the policy for the $(k+1)$-th iteration is obtained via invoking the policy improvement operator,
i.e.,
$\pi_{k+1} \doteq \pi_{w_k}$.
\citet{perkins2003convergent} show that
\begin{align}
  \label{eq doina}
  \textstyle \norm{\pi_{A_{\hat \pi}^{-1}b_{\hat \pi}} - \pi_{A_{\pi'}^{-1}b_{\pi'}}} \leq \fO\left(L_\pi\right) \norm{\hat \pi - \pi'},
\end{align}
where the policies $\hat \pi$ and $\pi'$ should be interpreted as vectors in $\R^{\nsa}$ whose $(s,a)$-indexed element is $\pi(a|s)$ when computing $\norm{\hat \pi - \pi'}$
and $L_\pi$ denotes the Lipschitz constant of the policy improvement operator $\pi_w$,
i.e., $\forall s, a$,
\begin{align}
  \abs{\pi_{w_1}(a|s) - \pi_{w_2}(a|s)} \leq L_\pi \norm{w_1 - w_2}.
\end{align}
Consequently,
if $L_\pi$ is sufficiently small,
the function $x \to \pi_{A^{-1}_{x} b_{x}}$,
which maps $\pi_k$ to $\pi_{k+1}$,
becomes a contraction.
Banach's fixed point theorem then confirms the convergence of fitted SARSA.
From the definition of $A_\pi$ and $b_\pi$ in~\eqref{eq source of error function},
it is easy to see that~\eqref{eq doina} can also be expressed as
\begin{align}
  \norm{\pi_{A_{\hat \pi}^{-1}b_{\hat \pi}} - \pi_{A_{\pi'}^{-1}b_{\pi'}}} \leq \fO\left(L_\pi \norm{r}\right) \norm{\hat \pi - \pi'}
\end{align}
because $r$ is a multiplier in the definition of $b_\pi$. 
Hence,
for any $L_\pi$, 
if the magnitude of the reward $\norm{r}$ is small enough,
the function $x \to \pi_{A^{-1}_x b_x}$ is contractive
and fitted SARSA remains convergent.
Nevertheless,
\citet{perkins2003convergent} share the same spirit as \citet{gordon2001reinforcement} by holding the policy fixed for sufficiently (possibly infinitely) many steps
to wait for the policy evaluation to complete.

When it comes to standard linear SARSA that updates the policy every time step,
\citet{melo2008analysis} consider, for a fixed point $w_*$,
\begin{align}
  \textstyle C_{w_*} \doteq \sup_{w} \norm{A_{\pi_w} - A_{\pi_{w_*}}} + \sup_{w \neq w_*} \frac{\norm{b_{\pi_w} - b_{\pi_{w_*}}}}{\norm{w - w_*}}.
\end{align}
They show that $C_{w_*} = \fO\left(L_\pi\right)$ and if $L_\pi$ is small enough such that
\begin{align}
  \label{eq melo matrix}
  A_{\pi_{w_*}} + C_{w_*} I 
\end{align}
is negative definite,
SARSA converge to $w_*$.
The convergence of SARSA in \citet{perkins2003convergent,melo2008analysis} does not require the projection operator 
(i.e., $C_\Gamma = \infty$) but is only asymptotic,
\citet{zou2019finite} further provide a convergence rate of  SARSA using some $C_\Gamma < \infty$,
assuming
$L_\pi$ is small enough such that
\begin{align}
  \label{eq zou matrix}
  A_{\pi_{w_*}} + \fO\left(L_\pi \left(r_{max} + 2 x_{max} C_\Gamma\right)\right) I
\end{align}
is negative definite.
It is easy to see that
if $L_\pi$ is not small enough,
neither~\eqref{eq melo matrix} nor~\eqref{eq zou matrix} can be guaranteed to be negative definite
no matter how small $\norm{r}$ is.
This is because the $\sup_{w} \norm{A_{\pi_w} - A_{\pi_{w_*}}}$ term in $C_{w_*}$ and the $2x_{max}C_\Gamma$ term in~\eqref{eq zou matrix} are independent of $r$.
In other words, unlike \citet{perkins2003convergent},
where the requirement for a sufficiently small $L_\pi$ is not necessary when the magnitude of the reward is enough,
a sufficiently small $L_\pi$ is an essential requirement for the analysis of \citet{melo2008analysis,zou2019finite}.
We,
however,
note that requring a finite $C_\Gamma$ and a sufficiently small $L_\pi$ can be restrictive. 
To see this,
consider the $\epsilon$-softmax policy in~\eqref{eq epsilon softmax policy}.
On the one hand,
for this parameterization,
we have $L_\pi = \frac{1 - \epsilon}{\iota}$.
For $L_\pi$ to be sufficiently small,
we have to ensure the temperature $\iota$ to be sufficiently large.
On the other hand,
a finte $C_\Gamma$ ensures that the action value estimate $q(s, a)$ inside $\exp(\cdot)$ is bounded.
Combining the two facts together,
it is easy to see that the $\epsilon$-softmax policy cannot be much different from a uniformly random policy.
Or more formally speaking,
the probability of any action is at most 
\begin{align}
  \frac{\epsilon}{\na} + \frac{1-\epsilon}{1 + \left(\na -1 \right) \exp\left(- \frac{2x_{max}C_\Gamma}{\iota}\right)}.
\end{align}
When $\iota$ is too large, 
the above probability approaches $\frac{1}{\na}$,
making it hard to ensure sufficient exploitation.
The behavior of SARSA with a large $L_\pi$,
however,
remains an open problem.



\section{Stochastic Approximation with Rapidly Changing Markov Chains}
\label{sec sa}
To prepare us for the analysis of SARSA,
we show, in this section,
a convergence rate (to a bounded region) of a generic stochastic approximation algorithm.
More precisely,
we consider the following iterative updates:
\begin{align}
    \label{eq sa iterates}
    w_{t+1} \doteq \Gamma\left(w_t + \alpha_t \left(F_{\theta_t}(w_t, Y_t) - w_t\right)\right).
\end{align}
Here $\qty{w_t \in \R^K}$ are the iterates generated by the stochastic approximation algorithm,
$\qty{Y_t}$ is a sequence of random variables evolving in a finite space $\fY$,
$\qty{\theta_t \in \R^L}$ is a sequence of random variables controlling the transition of $\qty{Y_t}$,
$F_\theta$ is a function from $\R^K \times \fY$ to $\R^K$ parameterized by $\theta$,
and $\Gamma$ is the projection operator defined in~\eqref{eq projection}.
Importantly,
we consider the setting where 
\begin{align}
  \forall t, \theta_t \equiv w_t.
\end{align}
In other words,
there is only a single iterate in our setting.
To ease presentation, we use $\qty{w_t}$ and $\qty{\theta_t}$ to denote the same quantity. This emphasizes their different roles as the iterates of interest and as the controller of the transition kernel.

Our analysis is a natural extension of \citet{chen2021lyapunov} and \citet{zhang2021global}
but has significant differences to theirs.
\citet{chen2021lyapunov} consider a time-homogeneous Markov chain (i.e., $\forall t, \theta_t \equiv \theta_0$).
Consequently,
their results are naturally applicable to policy evaluation problems.
\citet{zhang2021global} consider a time-inhomogeneous Markov chain,
where the iterates $\qty{w_t}$ are \emph{different} from the sequence $\qty{\theta_t}$.
More importantly,
\citet{zhang2021global} assume that $\qty{\theta_t}$ changes sufficiently slowly,
i.e.,
there exists another sequence $\qty{\beta_t}$ such that
\begin{align}
  \norm{\theta_{t+1} - \theta_t} = \fO\left(\beta_t\right), \,\lim_{t\to\infty} \frac{\beta_t}{\alpha_t} = 0.
\end{align}
This is the classical two-timescale setting (see, e.g., \citet{borkar2009stochastic}) and their analysis naturally applies to actor-critic algorithms \citep{konda2000actor} with
$\qty{w_t}$ and $\qty{\theta_t}$ interpreted as critic and actor respectively.
We
instead
consider the setting where $\forall t, \theta_t = w_t$.
In other words,
the time-inhomogeneous Markov chain we consider changes \emph{rapidly},
which is the main challenge of our analysis.
As a consequence, we
introduce the projection operator $\Gamma$,
not required in \citet{chen2021lyapunov,zhang2021global}.
The price is that we only show convergence to a bounded region
while \citet{chen2021lyapunov,zhang2021global} show convergence to points.
Convergence to a bounded region is,
however,
sufficient for our purpose of understanding the behavior of SARSA 
since it matches what practitioners have observed. 
Furthermore,
we believe our analysis might be applicable to other RL algorithms and might also have independent interest beyond RL.
We now state our assumptions.
\begin{assumption}
  \label{assu makovian}
  (Time-inhomogeneous Markov chain)
  There exists a family of parameterized transition matrices $\Lambda_P \doteq \qty{P_{\theta} \in \R^{|\fY| \times |\fY|} | \theta \in \R^L}$ such that
\begin{align}
  \label{eq markov}
    \Pr(Y_{t+1} = y) = P_{\theta_{t+1}}(Y_t, y).
\end{align}
\end{assumption}
\begin{assumption}
    \label{assu uniform ergodicity}
    (Uniform ergodicity)
Let $\bar \Lambda_P$ be the closure of $\Lambda_P$.
For any $P \in \bar \Lambda_P$,
the chain induced by $P$ is ergodic.
We use $d_\theta$ to denote the invariant distribution of the chain induced by $P_\theta$.
\end{assumption}
\noindent
Those two assumptions are identical to those of \citet{zhang2021global}.
Assumption~\ref{assu makovian} states that the random process $\qty{Y_t}$ is a time-inhomogeneous Markov chain.
Assumption~\ref{assu uniform ergodicity} states the ergodicity of the Markov chains.
Assumption~\ref{assu uniform ergodicity} is also used in the analysis of RL algorithms in both on-policy \citep{marbach2001simulation} and off-policy \citep{zhang2021global,zhang2021breaking} settings.
We later show how SARSA($\lambda$) can trivially fulfill Assumption~\ref{assu uniform ergodicity}.
The uniform ergodicity in Assumption~\ref{assu uniform ergodicity} immediately implies uniform mixing.
\begin{restatable}{lemma}{uniformmixing}
  \label{lem uniform mixing}
  (Lemma 1 of \citet{zhang2021global})
Let Assumption~\ref{assu uniform ergodicity} hold. 
Then, there exist constants $C_M > 0$ and $\tau \in (0, 1)$,
independent of $\theta$,
such that for any $n > 0$,
\begin{align}
    \label{eq uniform mixing}
    \textstyle \sup_{y, \theta} \sum_{y'} \abs{P^n_\theta(y, y') -  d_\theta(y')} \leq C_M \tau^n.
\end{align}
\end{restatable}
\noindent
As noted in \citet{zhang2021global},
the uniform mixing property in Lemma~\ref{lem uniform mixing} is usually a direct technical assumption in previous works (e.g., \citet{zou2019finite,wu2020finite}).

\begin{assumption}
    \label{assu uniform contraction}
(Uniform pseudo-contraction)
Let 
\begin{align}
    \bar F_\theta(w) &\doteq \sum_{y \in \fY} d_\theta(y) F_\theta(w, y), \\
    f^\alpha_\theta(w) &\doteq w + \alpha \left(\bar F_\theta(w) - w\right).
\end{align}
Then,
\begin{enumerate}[(i)]
  \item For any $\theta$, $\bar F_\theta$ has a unique fixed point, i.e., there exists a unique $w^*_\theta$ such that
  \begin{align}
    \bar F_\theta(w^*_\theta) = w^*_\theta.
  \end{align}
  \item There exists a constant $\bar \alpha > 0$ such that for all $\alpha \in (0, \bar \alpha)$,
  $f^\alpha_\theta$ is a uniform pseudo-contraction,
  i.e.,
  there exists a constant $\kappa_\alpha \in (0, 1)$ (depending on $\alpha$),
  such that for all $\theta, w$,
  \begin{align}
    \norm{f^\alpha_\theta(w) - w^*_\theta} \leq \kappa_\alpha \norm{w - w^*_\theta}.
  \end{align}
\end{enumerate}
\end{assumption}
Assumption~\ref{assu uniform contraction} is another difference from \citet{chen2021lyapunov,zhang2021global}.
Namely,
\citet{chen2021lyapunov,zhang2021global} require $\bar F_\theta$ to be a contraction 
while we only require $f_\theta^\alpha$ to be a pseudo-contraction.
It is easy to see that the contraction of $\bar F_\theta$ immediately implies the pseudo-contraction of $f_\theta^\alpha$ but not in the opposite direction.
In other words,
our assumption is weaker.

\begin{assumption}
  \label{assu regularization}
(Continuity and boundedness)
There exist constants $L_F, L_F', L_F'', U_F, U_F', U_F'', L_w, U_w, L_P$ such that for any $w, w', y, y', \theta, \theta'$,
\begin{enumerate}[(i).]
  \item $\norm{F_{\theta}(w, y) - F_{\theta}(w', y)} \leq L_F \norm{w - w'}$ 
  \item $\norm{F_{\theta}(w, y) - F_{\theta'}(w, y)} \leq L_F' \norm{\theta - \theta'} (\norm{w} + U_F')$
  \item $\norm{F_{\theta}(0, y)} \leq U_F$
  \item $\norm{\bar F_{\theta}(w) - \bar F_{\theta'}(w)} \leq L_F'' \norm{\theta - \theta'} (\norm{w} + U_F'')$ 
  \item $\norm{w^*_{\theta} - w^*_{\theta'}} \leq L_w \norm{\theta - \theta'}$ 
  \item $\sup_{\theta} \norm{w^*_{\theta}} \leq U_w $
  \item $\abs{ P_{\theta}(y, y') - P_{\theta'}(y, y') } \leq L_P \norm{\theta - \theta'}$
\end{enumerate}
\end{assumption}
Assumption~\ref{assu regularization} states some regularity conditions for the functions we consider and is identical to that of \citet{zhang2021global}.

\begin{assumption}
  \label{assu projection}
  (Projection)
  \begin{enumerate}[(i).]
    \item $\norm{w_0} \leq C_\Gamma, U_w \leq C_\Gamma$.
    \item For any $\theta, w, y$, we have
    \begin{align}
      P_{\theta} &= P_{\Gamma(\theta)}, \, F_{\theta}(w, y) = F_{\Gamma(\theta)}(w, y), \\
      w^*_\theta &= w^*_{\Gamma(\theta)}.
    \end{align}
  \end{enumerate}
\end{assumption}
Assumption~\ref{assu projection} requires that some of the functions we consider are invariant to the projection operator.
We will later show that SARSA($\lambda$) trivially satisfies this assumption.

\begin{assumption}
  \label{assu sa lr}
  The learning rates $\qty{\alpha_t}$ have the form 
  \begin{align}
    \textstyle \alpha_t \doteq \frac{c_\alpha}{(t_0 + t)^{\epsilon_\alpha}},
  \end{align}
  where $c_\alpha > 0, \epsilon_\alpha \in (0, 1], t_0 > 0$ are constants to be tuned.
\end{assumption}
Assumption~\ref{assu sa lr} is just one of many possible forms of learning rates;
we use this particular one to ease presentation.
Importantly, the learning rates $\qty{\alpha_t}$ here do \emph{not} verify the Robbins-Monro's condition \citep{robbins1951stochastic} when $\epsilon_\alpha \leq 0.5$,
neither do the learning rates in \citet{wu2020finite,chen2021lyapunov}.

We now present our analysis.
Given the sequences $\qty{\theta_t}$ (i.e., $\qty{w_t}$) and $\qty{Y_t}$ in \eqref{eq sa iterates},
we define an auxiliary sequence $\qty{u_t}$ as
\begin{align}
  u_0 \doteq& w_0, \\
  \label{eq sa iterates transformed}
  u_{t+1} \doteq& \Gamma(u_t) + \alpha_t (F_{\theta_t}(\Gamma(u_t), Y_t) - \Gamma(u_t)).
\end{align}
\begin{lemma}
  \label{lem transform}
  Let Assumption~\ref{assu projection} hold.
  Then for any $t$, $w_t = \Gamma(u_t)$.
\end{lemma}
\begin{proof}
  It follows immediately from induction.
\end{proof}
Intuitively,
$\qty{u_t}$ is simply the pre-projection version of $\qty{w_t}$.
We are interested in $\qty{u_t}$ 
because it has the following nice property.
\begin{restatable}{theorem}{thmsaconvergence}
    \label{thm sa convergence}
    Let Assumptions~\ref{assu makovian} - \ref{assu sa lr} hold.
    If $t_0$ is sufficiently large,
    then the iterates $\qty{u_t}$ generated by \eqref{eq sa iterates transformed} satisfy
    \begin{align}
      &\E\left[\norm{u_{t+1} - w^*_{\theta_{t+1}}}^2\right] \\
      \leq& \left(1 - 2\left(1 - \kappa_{\alpha_t} - \fO\left(\alpha_t^2 \log^2 \alpha_t \right) \right)\right) \E\left[\norm{\Gamma(u_{t}) - w^*_{\theta_{t}}}^2\right] \\
      &+ 2L_w L_\theta \alpha_t  \E\left[\norm{\Gamma(u_{t}) - w^*_{\theta_{t}}}\right] + \fO\left(\alpha_t^2 \log^2 \alpha_t\right),
    \end{align}
    where $L_{\theta} \doteq U_F + (L_F + 1)C_\Gamma$.
\end{restatable}
\noindent See Section~\ref{sec proof thm sa convergence} for the proof of Theorem~\ref{thm sa convergence},
where 
the constants hidden by $\fO\left(\cdot\right)$ 
and how large $t_0$ is are also explicitly documented.
Theorem~\ref{thm sa convergence} gives a recursive form of some error terms.
We, however, cannot go further unless we have the domain knowledge of $\kappa_\alpha$.

\begin{restatable}{corollary}{thmsasingleweight}
  \label{thm sa single weight}
  Let Assumptions~\ref{assu makovian} - \ref{assu sa lr} hold. 
  Assume $\kappa_\alpha = \sqrt{1 - \eta \alpha}$ for some positive constant $\eta > 0$.
  If $t_0$ is sufficiently large,
  then the iterates $\qty{w_t}$ generated by~\eqref{eq sa iterates} satisfy
  \begin{align}
  &\E\left[\norm{w_{t} - w^*_{w_{t}}}^2\right] = \frac{72L_w^2 L_\theta^2}{\eta^2} \\
  +& \begin{cases} 
    \fO\left(t^{-\frac{\eta c_\alpha}{3}} \log^2 t\right), &\epsilon_\alpha = 1, \eta c_\alpha \in (0, 3) \\
    \fO\left(\frac{\log^3 t}{t}\right), & \epsilon_\alpha = 1, \eta c_\alpha = 3 \\
    \fO\left(\frac{\log^2 t}{t}\right), & \epsilon_\alpha = 1, \eta c_\alpha \in (3, \infty) \\
    \fO\left(\frac{\log^2 t}{t^{\epsilon_\alpha}}\right), & \epsilon_\alpha \in (0, 1)
  \end{cases}.
\end{align}
\end{restatable}
\noindent See Section~\ref{sec proof of thm sa single weight} for the proof of Corollary~\ref{thm sa single weight},
as well as the constants hidden by $\fO\left(\cdot\right)$ and how large $t_0$ is.
Due to the projection operator,
we have the bound
\begin{align}
  \textstyle \E\left[\norm{w_{t} - w^*_{w_{t}}}^2\right] \leq 4C_\Gamma^2.
\end{align}
So Corollary~\ref{thm sa single weight} is informative only if  
\begin{align}
  \textstyle \frac{72L_w^2 L_\theta^2}{\eta^2} \leq 4C_\Gamma^2.
\end{align}
This is where we need more domain knowledge and the analysis in the next section provides an example.

\section{SARSA with Linear Function Approximation}
\label{sec sarsa}

We first analyze SARSA and expected SARSA with the following assumptions.
\begin{assumption}
  \label{assu lipschitz mu}
  (Lipschitz continuity) There exists $L_\pi > 0$ such that $\forall w, w', a, s$,
  \begin{align}
    \norm{\pi_w(a|s) - \pi_{w'}(a|s)} &\leq L_\pi \norm{w - w'}.
  \end{align}
\end{assumption}
\begin{assumption}
  \label{assu mu uniform ergodicity}
  (Uniform ergodicity)
  Let $\bar \Lambda_\pi$ be the closure of $\qty{\pi_w \mid w \in \R^\nsa}$.
  For any $\pi \in \bar \Lambda_\pi$,
  the chain induced by $\pi$ is ergodic and $\pi(a|s) > 0$.
\end{assumption}
\begin{assumption}
  \label{assu full rank}
  (Linear independence)
  The feature matrix $X$ has full column rank.
\end{assumption}
Assumption~\ref{assu lipschitz mu} is also used in \citet{perkins2003convergent,melo2008analysis,zou2019finite}.
As noted by \citet{zhang2021global},
Assumption~\ref{assu mu uniform ergodicity} is easy to fulfill especially
when the chain induced by a uniformly random policy is ergodic,
which we believe is a fairly weak assumption.
An example policy satisfying those two assumptions is the $\epsilon$-softmax policy in~\eqref{eq epsilon softmax policy} with any $\epsilon \in (0, 1]$,
provided that the chain induced by a uniformly random policy is ergodic.
Assumption~\ref{assu full rank} is standard in the literature regarding RL with linear function approximation (see, e.g., \citet{tsitsiklis1997analysis}).

As discussed in~\eqref{eq source of error function},
if SARSA, as well as expected SARSA, converged to some vector $w_*$,
that vector would verify
\begin{align}
  w_* = A_{\pi_{w_*}}^{-1} b_{\pi_{w_*}}.
\end{align}
This inspires us to define the error function
\begin{align}
  &e(w) \\
  \textstyle \doteq& \norm{w - \left(X^\top D_{\pi_{\Gamma(w)}} \left(\gamma P_{\pi_{\Gamma(w)}} - I\right)X\right)^{-1}X^\top D_{\pi_{\Gamma(w)} }r}^2
\end{align}
to study the behavior of SARSA.
Here we have included the projection operator $\Gamma$ in the definition of $e(w)$ because
we use a finite $C_\Gamma$ in
Algorithm~\ref{alg sarsa lambda}.
\begin{restatable}{theorem}{propcriticconvergence}
  \label{thm sarsa convergence}
  Let Assumptions~\ref{assu sa lr} and~\ref{assu lipschitz mu} -~\ref{assu full rank} hold.
  Assume $\norm{X} = 1$ and $r_{max}$ is not so large such that
  \begin{align}
    L_w \doteq \fO\left(L_\pi r_{max}\right) < 1.
  \end{align}
  Assume $C_\Gamma$ is large enough such that 
  \begin{align}
    U_w \doteq \fO\left(r_{max}\right) \leq C_\Gamma.
  \end{align}
  Let $t_0$ be sufficiently large.
  Then the iterates $\qty{w_t}$ generated by Algorithm~\ref{alg sarsa lambda} satisfy
  \begin{align}
    &\E\left[\norm{w_t - w_*}\right] = \frac{6\sqrt{2}L_w \left(1+4C_\Gamma\right)}{\eta (1-L_w)} \\
    &+ \begin{cases} 
    \fO\left(t^{-\frac{\eta c_\alpha}{6}} \log t\right), & \epsilon_\alpha = 1, \eta c_\alpha \in (0, 3) \\
    \fO\left(t^{-\frac{1}{2}} \log^{\frac{3}{2}}t \right), & \epsilon_\alpha = 1, \eta c_\alpha = 3 \\
    \fO\left(t^{-\frac{1}{2}} \log t\right), & \epsilon_\alpha = 1, \eta c_\alpha \in (3, \infty) \\
    \fO\left(t^{-\frac{\epsilon_\alpha}{2}} \log t\right), & \epsilon_\alpha \in (0, 1)
    \end{cases},
  \end{align}
  where $\eta$ is a positive constant and $w_*$ is the unique vector such that $e(w_*) = 0$.
\end{restatable}
\noindent
We prove Theorem~\ref{thm sarsa convergence} mainly by invoking Corollary~\ref{thm sa single weight}.
See Section~\ref{sec proof thm sarsa convergence} for more details.
The exact expressions of $L_w, U_w, \eta$ are detailed in~\eqref{eq exact lw},~\eqref{eq exact uw},~\eqref{eq exact eta} in the proof,
all of which are independent of $C_\Gamma$.

\paragraph*{Significance of Theorem~\ref{thm sarsa convergence}.}

Theorem~\ref{thm sarsa convergence} ensures that the iterates $\qty{w_t}$ converge, in expectation,
to a ball of size
\begin{align}
    \textstyle R_* \doteq \frac{6\sqrt{2}L_w \left(1+4C_\Gamma\right)}{\eta (1-L_w)},
\end{align}
centered at $w_*$.
Due to the use of projection,
one can also trivially claim that the iterates always stay in a ball of size $2C_\Gamma$, centered at $w_*$.
Thus Theorem~\ref{thm sarsa convergence} is informative only if
\begin{align}
    \label{eq informative}
    \textstyle \frac{R_*}{2C_\Gamma} = \frac{24\sqrt{2}L_w }{\eta (1-L_w)} \times \frac{1+4C_\Gamma}{4C_\Gamma} < 1.
\end{align}
The term $\frac{1+4C_\Gamma}{4C_\Gamma}$ monotonically decreases when the size of the ball for projection increases and eventually converges to $1$.
So this ratio is essentially determined by the first term $\frac{24\sqrt{2}L_w }{\eta (1-L_w)}$,
which can be arbitrarily small as long as $L_w$ is small.
In other words,
suppose 
$L_w$ is small enough,
no matter how large the ball used for projection is (practitioners usually use a very large ball for projection),
the iterates generated by linear SARSA asymptotically only visit a small portion of the ball.
The exact portion is determined by the relative magnitude of $L_w$ and other properties of the problem and can be arbitrarily small when $L_w$ is small enough.
Here we want to compare with $Q$-learning \citep{watkins1989learning} with linear function approximation.
As demonstrated in Baird's counterexample \citep{baird1995residual},
the $Q$-learning iterates can diverge to infinity.
This essentially means that if we apply a ball for projection in linear $Q$-learning,
the iterates can asymptotically visit every part of the ball.
By contrast,
Theorem~\ref{thm sarsa convergence} proves that the iterates generated by linear SARSA visit asymptotically only a possibly small portion of the ball,
on some problems.
In those problems,
Theorem~\ref{thm sarsa convergence},
to our best knowledge, 
is the first to characterize the fundamentally different behaviors between linear SARSA and linear $Q$-learning.
We regard this as the most important contribution of this work.
This difference demonstrates the challenge in off-policy learning compared with on-policy learning.

\paragraph*{Limitation of Theorem~\ref{thm sarsa convergence}.}
That being said,
there are a few things that Theorem~\ref{thm sarsa convergence} does not offer.

First, Theorem~\ref{thm sarsa convergence} is solely about the magnitude of the iterates and is not about the performance of the corresponding policy.
In other words,
this work does \emph{not} fully address the question raised by \citet{sutton1999open}.
Addressing that question requires to understand (a) how linear SARSA iterates behave asymptotically and (b) why such behavior generates good performance.
This work contributes to the former but does \emph{not} contribute to the latter.

Second, 
the ball of size $R_*$ specified in Theorem~\ref{thm sarsa convergence} is small only when compared with the the ball for projection on some problems.
If we compare $R_*$ with some other quantities of the problem,
it is indeed very large.
On some problems where $L_w$ is not small enough, 
the ball is also quite large.
In other words,
even for (a),
we only address it in some problems with a very coarse bound.
On problems where $L_w$ does not meet our condition,
Theorem~\ref{thm sarsa convergence} simply does not apply.
Theorem~\ref{thm sarsa convergence} is \emph{not} meant to be a general result that applies to all problems.
We, however, argue that this is still, to our knowledge, the best result regarding the question in \citet{sutton1999open}.

Third,
even if we step back to the tabular setting,
Theorem~\ref{thm sarsa convergence} is still convergence to only a ball instead of a fixed point.
This indicates the fundamental limit of the techniques employed therein.
This might seem disappointing at first glance but less so when taking a holistic view of the history of SARSA. 
The best result regarding tabular SARSA for control, 
to our knowledge, is \citet{singh2000convergence}, 
which proves that if the $\epsilon$ decays properly in the $\epsilon$-greedy policy, 
tabular SARSA converges to an optimal policy. 
That being said, to implement the decay, 
it is required to maintain a counter of the state-action visitations and the policy itself is non-Markovian in that it is a function of the current time step. 
When it comes to the canonical Markovian policies that depend only on current states, 
we still know nothing about the behavior of tabular SARSA for control.

\section{Related Work}

\paragraph*{Comparison with \citet{zou2019finite}.}
The most similar result to our work is \citet{zou2019finite}.
Assuming $L_\pi$ is small enough,
\citet{zou2019finite} give a finite sample analysis of the convergence of $\qty{w_t}$ to $w_*$.
Theorem~\ref{thm sarsa convergence} requires $L_w$ to be sufficiently small.
This holds when either $L_\pi$ or $r_{max}$ is sufficiently small.
In the case where $L_\pi$ is sufficiently small,
our result is indeed weaker than \citet{zou2019finite} because \citet{zou2019finite} show convergence to a point while we show convergence to only a region. 
Thus the only scenario that our result is preferred over \citet{zou2019finite} is when $L_\pi$ is large but $r_{max}$ is small.
In this scenario,
our result still apply but \citet{zou2019finite} do \emph{not} apply.
The reason is that in Theorem~\ref{thm sarsa convergence},
our condition is related to the product $L_\pi r_{max}$.
So the role of $L_\pi$ and $r_{max}$ is interchangeable and a small $L_\pi$ is only a sufficient condition for the product to be small.
However, 
\citet{zou2019finite}
essentially consider a condition
in the form of $L_\pi r_{max} + L_\pi C_\Gamma$ (see \eqref{eq zou matrix} or Assumption 2 in \citet{zou2019finite}).
For this summation to be sufficiently small,
a small $L_\pi$ is a necessary condition.
This means although our convergence is weaker than \citet{zou2019finite},
our result applies to much more problems than \citet{zou2019finite}.
In particular,
in the case where $L_\pi$ is large but $r_{max}$ is small,
our result applies but \citet{zou2019finite} do not.
More importantly,
due to the existence of $L_\pi C_\Gamma$,
$L_\pi$ in \citet{zou2019finite} is at most the order of $\fO\left(\frac{1}{C_\Gamma}\right)$.
Since practitioners usually use a very large, possibly infinite, ball for projection,
the $L_\pi$ in \citet{zou2019finite} has to be really small.
By contrast, our $L_w = \fO\left(L_\pi r_{max}\right)$ does \emph{not} have any dependence on $C_\Gamma$.
In other words, 
if $C_\Gamma$ is really large,
our result applies to much more $L_\pi$ than \citet{zou2019finite}.
We regard the broader settings as an improvement, and thus a contribution,
over \citet{zou2019finite}. 
To our best knowledge,
we do not know how to make \citet{zou2019finite} work in such broader settings.


A concurrent work \citep{gopalan2022approximate} prove that the iterates in their linear SARSA is bounded almost surely.
In particular,
they do not have projection but their result is only asymptotic without finite sample analysis.
The most significant restriction is that they require i.i.d. samples,
i.e.,
at each time step, the state $S_t$ is assumed to be sampled from the stationary distribution $d_{\pi_{w_t}}$ of the current policy $\pi_{w_t}$.
Since the policy $w_t$ changes rapidly every time step,
we argue that such i.i.d. samples are hard to obtain in practice.

Our results regarding the finite sample analysis of the general stochastic approximation algorithm in Section~\ref{sec sa} rely on the pseudo contraction property
and follow from \citet{chen2021lyapunov,zhang2021global}.
Another family of convergent results regarding stochastic approximation algorithms
is usually based on the analysis of the corresponding ordinary differential equations
(see, e.g., \citet{DBLP:books/sp/BenvenisteMP90,kushner2003stochastic,borkar2009stochastic}).
See \citet{chen2021lyapunov} and references therein for more details. 

SARSA is an extension of TD for control.
The convergence of linear TD,
which aims at estimating the value of a \emph{fixed} policy,
is an active research area,
see,
e.g.,
\citet{tsitsiklis1997analysis,dalal2018finiteaaai,lakshminarayanan2018linear,bhandari2018finite,srikant2019finite}.
Analyzing SARSA is more challenging than TD because the policy SARSA considers \emph{changes} every step. 

SARSA is an incremental and stochastic way to implement approximate policy iteration.
Other variants of approximate policy iteration include \citet{lagoudakis2003least,antos2008learning,farahmand2010error,lazaric2012finite,lazaric2016analysis}.

\section{Experiments}
\begin{figure}[h]
  \centering
  \includegraphics[width=0.3\textwidth]{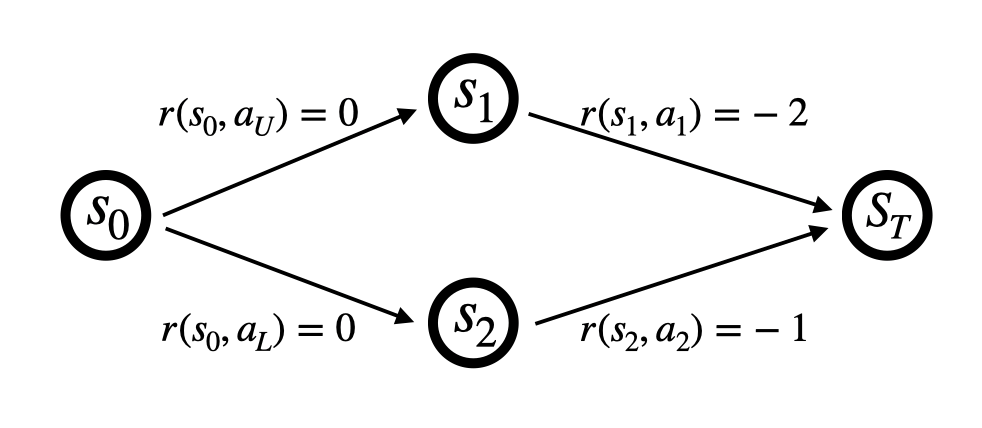}
  \caption{\label{fig mdp} 
  A diagnostic MDP from \citet{gordon1996chattering}. 
  The state $s_0$ is the initial state with two actions $a_U$ and $a_L$ available,
  both of which yield a 0 reward.
  At $s_1$,
  only one action $a_1$ is available,
  which yields a reward -2.
  At $s_2$,
  the action $a_2$ yields a reward -1.
  Both $a_1$ and $a_2$ leads to the terminal state $S_T$.}
\end{figure}
We use a diagnostic MDP from \citet{gordon1996chattering} (Figure~\ref{fig mdp}) to illustrate the chattering of linear SARSA.
\citet{gordon1996chattering} tested the $\epsilon$-greedy policy \eqref{eq epsilon greedy policy},
which is not continuous.
We further test the $\epsilon$-softmax policy~\eqref{eq epsilon softmax policy},
whose Lipschitz constant is inversely proportional to the temperature $\iota$.
When $\iota$ approaches $0$,
the $\epsilon$-softmax policy approaches the $\epsilon$-greedy policy.
We run Algorithm~\ref{alg sarsa lambda} in this MDP
with $C_\Gamma = \infty$,
i.e.,
there is no projection.
Following \citet{gordon1996chattering},
we set $\epsilon = 0.1, \gamma = 1.0$, and $\alpha_t = 0.01 \, \forall t$.
As discussed in \citet{gordon1996chattering},
using a smaller discount factor or a decaying learning rate only
slows down the chattering
but the chattering always occurs.
Following \citet{gordon1996chattering},
we use the following feature function:
\begin{align}
  x(s_0, a_U) &= [1, 0, 0]^\top, \, x(s_0, a_L) = [0, 1, 0]^\top, \\
  x(s_1, a_1) &= x(s_2, a_2) = [0, 0, 1]^\top.
\end{align}
In other words, 
it is essentially state aggregation.

\begin{figure}
  \centering
  \includegraphics[width=0.3\textwidth]{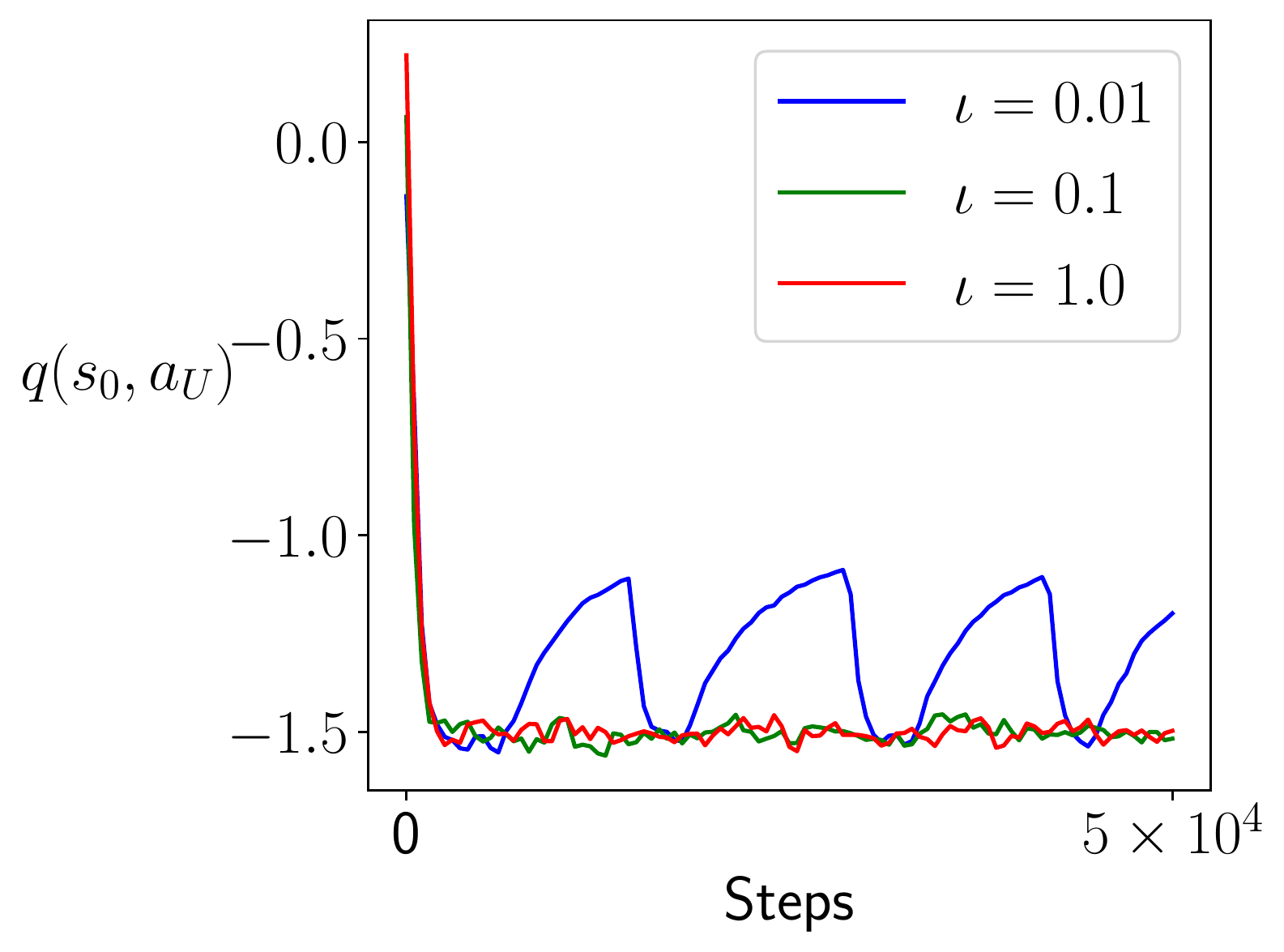}
  \caption{\label{fig chain0}
  The action value of $a_U$ during training under different temperatures.}
\end{figure}

\begin{figure}
  \centering
  \includegraphics[width=0.3\textwidth]{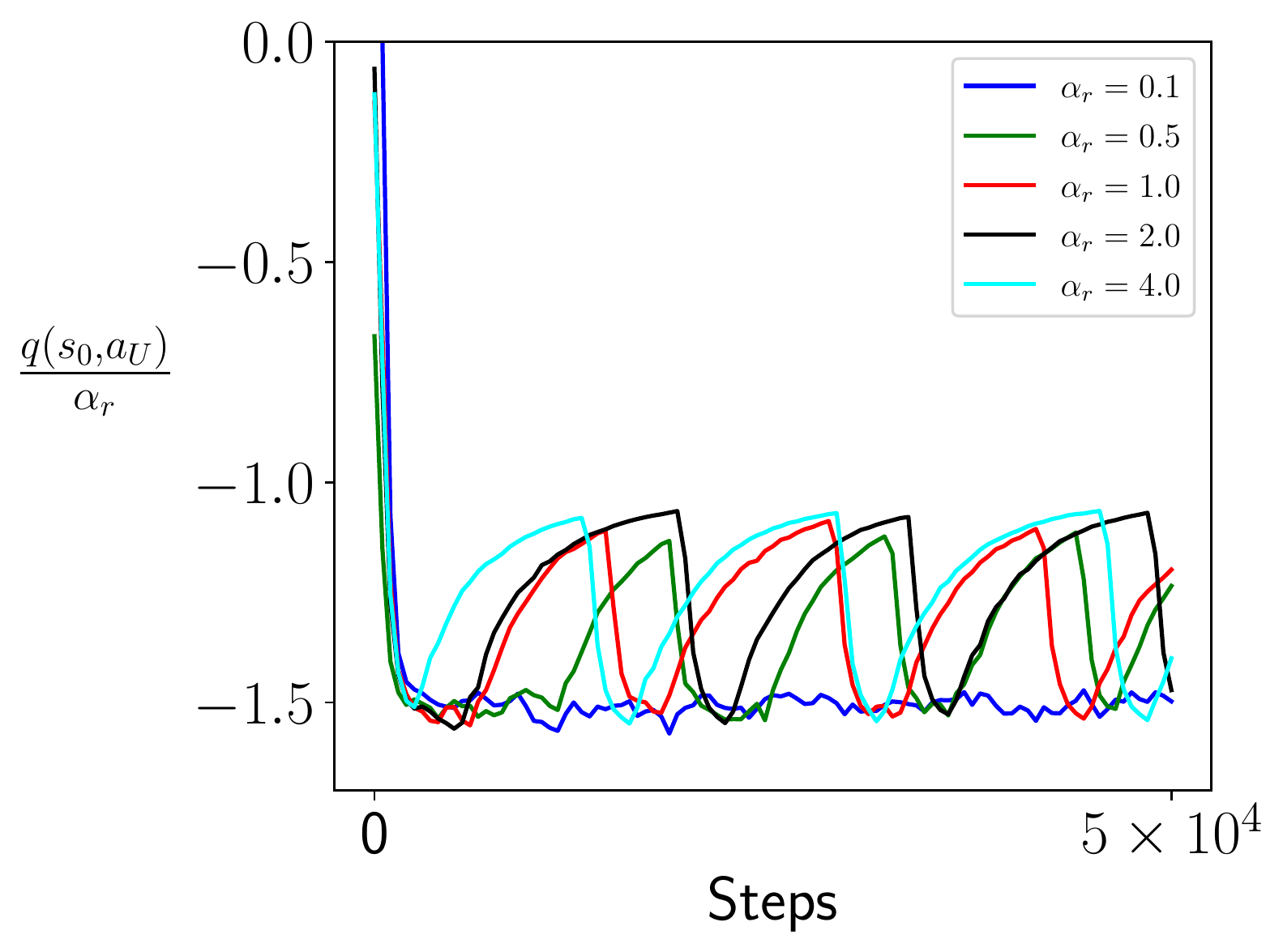}
  \caption{\label{fig chain1}
  The $\alpha_r$-normalized action value of $a_U$ during training with a fixed temperature $\iota = 0.01$.
  The reward of the MDP in Figure~\ref{fig mdp} is scaled via $\alpha_r$,
  e.g.,
  the reward for the action $a_1$ is now $-2\alpha_r$.
  }
\end{figure}

As shown in Figure~\ref{fig chain0},
when the temperature is small (i.e., $\iota = 0.01$),
linear SARSA chatters.
We further fix $\iota$ to be $0.01$ and 
test reward with different magnitudes. 
To this end, we multiply the reward with a multiplier $\alpha_r$.
We stress that this is just a simple way to get MDPs with rewards of different magnitudes.
It does \emph{not} mean that one should artificially scale the reward down when Theorem~\ref{thm sarsa convergence} does not apply.
As shown in Figure~\ref{fig chain1},
the chattering behavior disappears with $\alpha_r = 0.1$.
This suggests that our results might be improved such that when the magnitude of the rewards is small enough 
we can also achieve convergence to a fixed point,
instead of a bounded region.
We,
however,
leave this for future work.
When we set $\alpha_r=4.0$,
the iterates still only chatter but do not diverge.
This suggests that our requirement for $r_{max}$ might be only sufficient and not necessary.
We,
however,
leave the development of a necessary condition for future work.
All the curves in Figures~\ref{fig chain0} and~\ref{fig chain1} are from a single run.
Due to the randomness in the policy and the initialization of the weight, we find the peaks and valleys can sometimes average each other out
when we average over multiple runs. 

\section{Conclusion}
The behavior of linear SARSA is a long-standing open problem in the RL community.
Despite the progress made in this work,
there are still many open problems regarding the behavior of linear SARSA.
To name a few:
how does linear SARSA behave if the policy improvement operator is merely continuous but not Lipschitz continuous?
How does linear SARSA behave if both the Lipschitz constant and the magnitude of the rewards are not small?
Can we get a convergence rate without using any projection?
We hope this work can draw more attention to 
the convergence of linear SARSA,
arguably one of the most fundamental RL algorithms.

\section*{Acknowledgements}
We thank Shimon Whiteson and Nicolas Le Roux for their insightful comments. 
SZ is part of the Link Lab at the University of Virginia. RTdC and RL were part of MSR Montreal when this work was done.

\bibliography{ref.bib}
\bibliographystyle{icml2023}

\newpage
\appendix
\onecolumn

\section{Proofs of Section~\ref{sec sa}}
\subsection{Proof of Theorem \ref{thm sa convergence}}
\label{sec proof thm sa convergence}

\thmsaconvergence*
\begin{proof}
  We consider a Lyapunov function
  \begin{align}
    M(x) \doteq \frac{1}{2} \norm{x}^2.
  \end{align}
  It is well-known that for any $x, x'$,
  \begin{align}
      M(x') \leq M(x) + \indot{\nabla M(x)}{x' - x} + \frac{1}{2} \norm{x - x'}^2.
  \end{align}
  Using $x' = u_{t+1} - w^*_{\theta_{t+1}}$ and $x = \Gamma(u_t) - w^*_{\theta_t}$ in the above inequality and recalling the update~\eqref{eq sa iterates transformed}
  \begin{align}
      u_{t+1} =& \Gamma(u_t) + \alpha_t (F_{\theta_t}(\Gamma(u_t), Y_t) - \Gamma(u_t)) \\
      =& f_{\theta_t}^{\alpha_t}(\Gamma(u_t)) + \alpha_t \left(F_{\theta_t}(\Gamma(u_t), Y_t) - \bar F_{\theta_t}(\Gamma(u_t))\right)
  \end{align} 
  yield
  \begin{align}
    \label{eq expansion of M}
    &\frac{1}{2} \norm{u_{t+1} - w^*_{\theta_{t+1}}}^2 \\
    \leq& \frac{1}{2}\norm{\Gamma(u_t) - w^*_{\theta_t}}^2 + \indot{\Gamma(u_t) - w^*_{\theta_t}}{u_{t+1} - \Gamma(u_t) + w^*_{\theta_t} - w^*_{\theta_{t+1}}} \\
    &+ \frac{1}{2} \norm{u_{t+1} - \Gamma(u_t) + w^*_{\theta_t} - w^*_{\theta_{t+1}}}^2 \\
    =&\frac{1}{2}\norm{\Gamma(u_t) - w^*_{\theta_t}}^2 \\
    &+ \underbrace{\indot{\Gamma(u_t) - w^*_{\theta_t}}{w^*_{\theta_t} - w^*_{\theta_{t+1}}}}_{T_1} \\
    &+ \underbrace{\indot{\Gamma(u_t) - w^*_{\theta_t}}{f^{\alpha_t}_{\theta_t}(\Gamma(u_t)) - \Gamma(u_t)}}_{T_2} \\
    &+ \alpha_t \underbrace{\indot{\Gamma(u_t) - w^*_{\theta_t}}{F_{\theta_t}(\Gamma(u_t), Y_t) - \bar F_{\theta_t}(\Gamma(u_t))} }_{T_3} \\
    &+ \alpha_t^2\underbrace{\norm{F_{\theta_t}(\Gamma(u_t), Y_t) - \Gamma(u_t)}^2}_{T_5} \\
    &+ \underbrace{\norm{w^*_{\theta_t} - w^*_{\theta_{t+1}}}^2}_{T_6}.
  \end{align}
Here we do not have $T_4$ because the counterpart in \citet{zhang2021global} is now 0.
To further decompose $T_3$,
we define a function $\tau_\alpha$ of $\alpha$ as
\begin{align}
  \label{eq definition of tau alpha t}
    \tau_{\alpha} \doteq \min\qty{n \geq 0 \mid C_M \tau^n \leq \alpha},
\end{align}
where the constants $C_M$ and $\tau$ are given in Lemma~\ref{lem uniform mixing}.
In particular, 
$\tau_{\alpha_t}$ denotes the number of steps the chain needs to mix to an accuracy of $\alpha_t$.
It is easy to see 
\begin{align}
  \label{eq lr order}
  \tau_{\alpha} = \fO\left(-\log \alpha\right), \, \lim_{\alpha\to 0} \alpha\tau_\alpha = 0.
\end{align}
  We now decompose $T_3$ as
  \begin{align}
      T_3 =& \indot{\Gamma(u_t) - w^*_{\theta_t}}{F_{\theta_t}(\Gamma(u_t), Y_t) - \bar F_{\theta_t}(\Gamma(u_t))} \\
      =&\underbrace{\indot{\Gamma(u_t) - w^*_{\theta_t} - \left(\Gamma(u_{t-\tau_{\alpha_t}}) - w^*_{\theta_{t-\tau_{\alpha_t}}}\right) }{F_{\theta_t}(\Gamma(u_t), Y_t) - \bar F_{\theta_t}(\Gamma(u_t))}}_{T_{31}} \\
      &+\underbrace{\indot{\Gamma(u_{t-\tau_{\alpha_t}}) - w^*_{\theta_{t-\tau_{\alpha_t}}}}{F_{\theta_t}(\Gamma(u_t), Y_t) - F_{\theta_t}(\Gamma(u_{t- \tau_{\alpha_t}}), Y_t) + \bar F_{\theta_t}(\Gamma(u_{t- \tau_{\alpha_t}})) - \bar F_{\theta_t}(\Gamma(u_t))}}_{T_{32}}\\
      &+\underbrace{\indot{\Gamma(u_{t-\tau_{\alpha_t}}) - w^*_{\theta_{t-\tau_{\alpha_t}}}}{F_{\theta_t}(\Gamma(u_{t- \tau_{\alpha_t}}), Y_t) - \bar F_{\theta_t}(\Gamma(u_{t- \tau_{\alpha_t}}))}}_{T_{33}}.
  \end{align}
  We further decompose $T_{33}$ as
  \begin{align}
    T_{33} =& \indot{\Gamma(u_{t-\tau_{\alpha_t}}) - w^*_{\theta_{t-\tau_{\alpha_t}}}}{F_{\theta_t}(\Gamma(u_{t- \tau_{\alpha_t}}), Y_t) - \bar F_{\theta_t}(\Gamma(u_{t- \tau_{\alpha_t}}))} \\
    =& \underbrace{\indot{\Gamma(u_{t-\tau_{\alpha_t}}) - w^*_{\theta_{t-\tau_{\alpha_t}}}}{F_{\theta_{t - \tau_{\alpha_t}}}(\Gamma(u_{t- \tau_{\alpha_t}}), \tilde Y_t) - \bar F_{\theta_{t-\tau_{\alpha_t}}}(\Gamma(u_{t- \tau_{\alpha_t}}))}}_{T_{331}} + \\
    & \underbrace{\indot{\Gamma(u_{t-\tau_{\alpha_t}}) - w^*_{\theta_{t-\tau_{\alpha_t}}}}{F_{\theta_{t - \tau_{\alpha_t}}}(\Gamma(u_{t- \tau_{\alpha_t}}), Y_t) -F_{\theta_{t - \tau_{\alpha_t}}}(\Gamma(u_{t- \tau_{\alpha_t}}), \tilde Y_t)}}_{T_{332}} + \\
    & \underbrace{\indot{\Gamma(u_{t-\tau_{\alpha_t}}) - w^*_{\theta_{t-\tau_{\alpha_t}}}}{F_{\theta_{t}}(\Gamma(u_{t- \tau_{\alpha_t}}), Y_t) -F_{\theta_{t - \tau_{\alpha_t}}}(\Gamma(u_{t- \tau_{\alpha_t}}), Y_t)}}_{T_{333}} + \\
    & \underbrace{\indot{\Gamma(u_{t-\tau_{\alpha_t}}) - w^*_{\theta_{t-\tau_{\alpha_t}}}}{\bar F_{\theta_{t-\tau_{\alpha_t}}}(\Gamma(u_{t- \tau_{\alpha_t}})) - \bar F_{\theta_{t}}(\Gamma(u_{t- \tau_{\alpha_t}}))}}_{T_{334}}.
  \end{align}
  Here $\qty{\tilde Y_t}$ is an auxiliary chain inspired from \citet{zou2019finite}.
  Before time $t - \tau_{\alpha_t} - 1$,
  $\qty{\tilde Y_t}$ is exactly the same as $\qty{Y_t}$.
  After time $t - \tau_{\alpha_t} - 1$,
  $\tilde Y_t$ evolves according to the \emph{fixed} kernel $P_{\theta_{t-\tau_{\alpha_t}}}$
  while $Y_t$ evolves according the changing kernel $P_{\theta_{t - \tau_{\alpha_t}}}, P_{\theta_{k - \tau_{\alpha_t} + 1}}, \dots$.
  \begin{align}
    \qty{\tilde Y_t}&: \dots \to Y_{t-\tau_{\alpha_t}-1} \underbrace{\to}_{P_{\theta_{t-\tau_{\alpha_t}}}} Y_{t-\tau_{\alpha_t}} \underbrace{\to}_{P_{\theta_{t-\tau_{\alpha_t}}}} \tilde Y_{t-\tau_{\alpha_t}+1} \underbrace{\to}_{P_{\theta_{t-\tau_{\alpha_t}}}} \tilde Y_{t-\tau_{\alpha_t}+2} \to \dots \\
    \qty{Y_t}&: \dots \to Y_{t-\tau_{\alpha_t}-1} \underbrace{\to}_{P_{\theta_{t-\tau_{\alpha_t}}}} Y_{t-\tau_{\alpha_t}} \underbrace{\to}_{P_{\theta_{t-\tau_{\alpha_t}+1}}} Y_{t-\tau_{\alpha_t}+1} \underbrace{\to}_{P_{\theta_{t-\tau_{\alpha_t}+2}}} Y_{t-\tau_{\alpha_t}+2} \to \dots.
  \end{align}
  
  We are now ready to present bounds for each of the above terms.
  To begin,
  we define some shorthand:
  \begin{align}
    \label{eq shorthand a and b}
    A &\doteq 2 L_F + 1, \quad B \doteq U_F, \quad C \doteq AU_w + B + A + A(1 + U_F' + U_F''), \, \alpha_{i:j} \doteq \sum_{t=i}^j \alpha_t.
  \end{align}
  According to \eqref{eq lr order},
  we can select a sufficiently large $t_0$ such that
  \begin{align}
    \alpha_{t-\tau_{\alpha_t}, t-1} \leq \frac{1}{4A}
  \end{align}
  holds for all $t$.
  This condition is crucial for Lemma~\ref{lem bound of xk diff},
  which plays an important role in the following bounds.
  
  \begin{restatable}{lemma}{lemboundtone}
    \label{lem bound t1}
    (Bound of $T_1$)
    \begin{align}
        T_1 
        \leq L_w L_\theta \alpha_t \norm{\Gamma(u_t) - w^*_{\theta_t}}.
    \end{align}
  \end{restatable}
  \noindent
  The proof of Lemma~\ref{lem bound t1} is provided in Section~\ref{sec proof lem bound t1}.
  
  \begin{restatable}{lemma}{lemboundttwo}
    \label{lem bound t2}
  (Bound of $T_2$)
  \begin{align}
    T_2 \leq -(1 - \kappa_{\alpha_t}) \norm{\Gamma(u_t) - w^*_{\theta_t}}^2.
  \end{align}
  \end{restatable}
  \noindent
  The proof of Lemma~\ref{lem bound t2} is provided in Section~\ref{sec proof lem bound t2}.
  
  \begin{restatable}{lemma}{lemboundoftthreeone}
    \label{lem bound of t31}
    (Bound of $T_{31}$)
    \begin{align}
        T_{31} \leq 8(L_wL_\theta + 1)  \alpha_{t-\tau_{\alpha_t}, t-1} \left(A^2\norm{\Gamma(u_t) - w^*_{\theta_t}}^2 + C^2\right).
    \end{align}
  \end{restatable}
  \noindent
  The proof of Lemma~\ref{lem bound of t31} is provided in Section~\ref{sec proof lem bound t31}.
  
  \begin{restatable}{lemma}{lemboundtthreetwo}
    \label{lem bound t32}
    (Bound of $T_{32}$)
    \begin{align}
        T_{32} \leq 16 \alpha_{t-\tau_{\alpha_t}, t-1}(1 + L_wL_\theta \alpha_{t-\tau_{\alpha_t}, t-1}) \left( A^2\norm{\Gamma(u_t) - w^*_{\theta_t}}^2 + C^2 \right).
    \end{align}
  \end{restatable}
  \noindent
  The proof of Lemma~\ref{lem bound t32} is provided in Section~\ref{sec proof lem bound t32}.
  
  \begin{restatable}{lemma}{lemboundtthreethreeone}
    \label{lem bound t331}
    (Bound of $T_{331}$)
  \begin{align}
    \E\left[T_{331}\right]
    \leq & \frac{8 \alpha_t(1 + L_wL_\theta \alpha_{t-\tau_{\alpha_t}, t-1})}{A}\left(A^2 \E \left[\norm{\Gamma(u_t) - w^*_{\theta_t}}^2\right] + C^2\right).
  \end{align}
  \end{restatable}
  \noindent
  The proof of Lemma~\ref{lem bound t331} is provided in Section~\ref{sec proof lem bound t331}.
  
  \begin{restatable}{lemma}{lemboundtthreethreetwo}
    \label{lem bound t332}
    (Bound of $T_{332}$)
    \begin{align}
        \E\left[T_{332}\right]
        \leq \frac{8 \ny L_P L_\theta \sum_{j=t-\tau_{\alpha_t}}^{t-1}\alpha_{t-\tau_{\alpha_t}, j} (1 + L_wL_\theta \alpha_{t-\tau_{\alpha_t}, t-1})}{A} \left(A^2 \E\left[ \norm{\Gamma(u_t) - w^*_{\theta_t}}^2 \right] + C^2\right).
    \end{align}
  \end{restatable}
  \noindent
  The proof of Lemma~\ref{lem bound t332} is provided in Section~\ref{sec proof lem bound t332}.
  
  \begin{restatable}{lemma}{lemboundoftthreethreethree}
    \label{lem bound of t333}
    (Bound of $T_{333}$)
    \begin{align}
    T_{333} \leq \frac{4L_F' L_\theta \alpha_{t-\tau_{\alpha_t}, t-1} (1 + L_wL_\theta \alpha_{t-\tau_{\alpha_t}, t-1})}{A^2 }\left(A^2 \norm{\Gamma(u_t) - w^*_{\theta_t}}^2 + C^2\right).
    \end{align}
  \end{restatable}
  \noindent
  The proof of Lemma~\ref{lem bound of t333} is provided in Section~\ref{sec proof lem bound t333}.
  
  \begin{restatable}{lemma}{lemboundtthreethreefour}
    \label{lem bound t334}
    (Bound of $T_{334}$)
    \begin{align}
    T_{334} \leq \frac{4L_F'' L_\theta \alpha_{t-\tau_{\alpha_t}, t-1}(1 + L_wL_\theta \alpha_{t-\tau_{\alpha_t}, t-1})}{A^2 }\left(A^2 \norm{\Gamma(u_t) - w^*_{\theta_t}}^2 + C^2\right).
    \end{align}
  \end{restatable}
  \noindent
  The proof of Lemma~\ref{lem bound t334} is provided in Section~\ref{sec proof lem bound t334}.
  
  \begin{restatable}{lemma}{lemboundtfive}
    \label{lem bound t5}
    (Bound of $T_5$)
    \begin{align}
        T_5 
        \leq 2\left(A^2 \norm{\Gamma(u_t) - w^*_{\theta_t}}^2 + C^2\right).
    \end{align}
  \end{restatable}
  \noindent
  The proof of Lemma~\ref{lem bound t5} is provided in Section~\ref{sec proof lem bound t5}.

  \begin{restatable}{lemma}{lemboundtsix}
    \label{lem bound t6}
    (Bound of $T_6$)
    \begin{align}
        T_6 =  \norm{w^*_{\theta_t} - w^*_{\theta_{t+1}}}^2 \leq L_w^2L_\theta^2 \alpha_t^2.
    \end{align}
  \end{restatable}
  \noindent
  Lemma~\ref{lem bound t6} follows immediately from Lemma~\ref{lem accu learning rates}.
  
  We now assemble the bounds in Lemmas~\ref{lem bound t1} - \ref{lem bound t6} back into \eqref{eq expansion of M}.
  Define 
  \begin{align}
    L_{\alpha, t} \doteq& \left(\sum_{j=t-\tau_{\alpha_t}}^{t} \alpha_{t-\tau_{\alpha_t}, j} \right) (1 + L_w L_\theta \max\qty{1 , \alpha_{t-\tau_{\alpha_t}, t}}), \\
    C_0 \doteq& \max\qty{A^2, C^2} \max\qty{16, L_w^2 L_\theta^2, 8\ny L_P L_\theta, 4L_F'L_\theta, 4L_F''L_\theta},
  \end{align}
  Using $A > 1$ and Lemmas~\ref{lem bound of t31}~-~\ref{lem bound t6},
  it is easy to see 
  \begin{align}
    \E\left[T_{3}\right] \leq& C_0 L_{\alpha, t} \left(\E\left[\norm{\Gamma(u_t) - w^*_{\theta_t}}^2\right] + 1\right), \\
    \alpha_t\E\left[T_{3}\right] \leq& C_0 \alpha_t L_{\alpha, t}\left(\E\left[\norm{\Gamma(u_t) - w^*_{\theta_t}}^2\right] + 1\right), \\
    \E\left[T_5\right] \leq& C_0 \left(\E\left[\norm{\Gamma(u_t) - w^*_{\theta_t}}^2\right] + 1\right), \\
    \alpha_t^2\E\left[T_5\right] \leq& C_0 \alpha_t^2 \left(\E\left[\norm{\Gamma(u_t) - w^*_{\theta_t}}^2\right] + 1\right), \\
    \E\left[T_6\right] \leq& C_0 \alpha_t^2.
  \end{align}
  Then we have
  \begin{align}
    &\frac{1}{2} \E\left[\norm{u_{t+1} - w^*_{\theta_{t+1}}}^2\right] \\
    \leq&\frac{1}{2}\E\left[\norm{\Gamma(u_t) - w^*_{\theta_t}}^2\right] + L_w L_\theta \alpha_t \E\left[\norm{\Gamma(u_t) - w^*_{\theta_t}}\right] -(1 - \kappa_{\alpha_t}) \E\left[\norm{\Gamma(u_t) - w^*_{\theta_t}}^2\right]\\
    &+ C_0 \alpha_t L_{\alpha, t} \left(\E\left[\norm{\Gamma(u_t) - w^*_{\theta_t}}^2\right] + 1\right) + C_0 \alpha_t^2 \left(\E\left[\norm{\Gamma(u_t) - w^*_{\theta_t}}^2\right] + 1\right)\\
    &+ C_0 \alpha_t^2,
  \end{align}
  implying
  \begin{align}
    \E\left[\norm{u_{t+1} - w^*_{\theta_{t+1}}}^2\right] \leq& \left(1 - 2\left(1 - \kappa_{\alpha_t} - C_0 \alpha_t L_{\alpha, t} - C_0 \alpha_t^2 \right)\right) \E\left[\norm{\Gamma(u_{t}) - w^*_{\theta_{t}}}^2\right] \\
    &+ 2 L_w L_\theta \alpha_t  \E\left[\norm{\Gamma(u_{t}) - w^*_{\theta_{t}}}\right] + 2 C_0 \alpha_t L_{\alpha, t} + 4C_0 \alpha_t^2.
  \end{align}
  Observing that
  \begin{align}
    L_{\alpha, t} = \fO\left(\alpha_t \log^2 \alpha_t \right)
  \end{align}
  then completes the proof.
\end{proof}

\subsection{Proof of Corollary~\ref{thm sa single weight}}
\label{sec proof of thm sa single weight}
\thmsasingleweight*
\begin{proof}
  According to Theorem~\ref{thm sa convergence},
  we have
  \begin{align}
    \E\left[\norm{u_{t+1} - w^*_{\theta_{t+1}}}^2\right] \leq& \left(1 - 2\left(1 - \kappa_{\alpha_t} - C_0 \alpha_t L_{\alpha, t} - C_0 \alpha_t^2 \right)\right) \E\left[\norm{\Gamma(u_{t}) - w^*_{\theta_{t}}}^2\right] \\
    &+ 2L_w L_\theta \alpha_t  \E\left[\norm{\Gamma(u_{t}) - w^*_{\theta_{t}}}\right] + 2C_0 \alpha_t L_{\alpha, t} + 4C_0 \alpha_t^2.
  \end{align}
  Since $L_{\alpha, t} = \fO\left(\alpha_t \log^2 \alpha_t\right)$,
  we conclude that there exists a constant $C_1 > 0$ such that
  \begin{align}
    \E\left[\norm{u_{t+1} - w^*_{\theta_{t+1}}}^2\right] \leq& \left(1 - 2\left(1 - \sqrt{1 - \eta \alpha_t} - C_1 \alpha_t^2 \log^2 \alpha_t \right)\right) \E\left[\norm{\Gamma(u_{t}) - w^*_{\theta_{t}}}^2\right] \\
    &+ 2L_w L_\theta \alpha_t  \E\left[\norm{\Gamma(u_{t}) - w^*_{\theta_{t}}}\right] + C_1 \alpha_t^2 \log^2 \alpha_t.
  \end{align}
  When $t_0$ is sufficiently large,
  we have $\forall t$,
  \begin{align}
    1 - 2\left(1 - \sqrt{1 - \eta \alpha_t} - C_1 \alpha_t^2 \log^2 \alpha_t \right) > 0.
  \end{align}
  Using 
  \begin{align}
    \norm{\Gamma(u_t) - w^*_{\theta_t}} = \norm{\Gamma(u_t) - \Gamma(w^*_{\theta_t})} \leq \norm{u_t - w^*_{\theta_t}}
  \end{align}
  then yields
  \begin{align}
    &\E\left[\norm{u_{t+1} - w^*_{\theta_{t+1}}}^2\right] \\
    \leq& \left(1 - 2\left(1 - \sqrt{1 - \eta \alpha_t} - C_1 \alpha_t^2 \log^2 \alpha_t \right)\right) \E\left[\norm{u_{t} - w^*_{\theta_{t}}}^2\right] \\
    &+ 2L_w L_\theta \alpha_t  \E\left[\norm{u_{t} - w^*_{\theta_{t}}}\right] + C_1 \alpha_t^2 \log^2 \alpha_t \\
    \leq& \left(1 - 2\left(1 - \sqrt{1 - \eta \alpha_t} - C_1 \alpha_t^2 \log^2 \alpha_t \right)\right) \E\left[\norm{u_{t} - w^*_{\theta_{t}}}^2\right] \\
    &+ 2L_w L_\theta \alpha_t  \sqrt{\E\left[\norm{u_{t} - w^*_{\theta_{t}}}^2\right]} + C_1 \alpha_t^2 \log^2 \alpha_t
    \qq{(Jensen's inequality)}.
  \end{align}
  Since 
  \begin{align}
    \lim_{\alpha \to 0} \frac{1 - \sqrt{1 - \eta \alpha}}{\frac{\eta}{2}\alpha} = 1,
  \end{align}
  we conclude that when $t_0$ is sufficiently large, $\forall t$,
  \begin{align}
    1 - \sqrt{1 - \eta \alpha_t} - C_1 \alpha_t^2 \log^2 \alpha_t \geq \frac{\eta}{3}\alpha_t - C_1 \alpha_t^2 \log^2 \alpha_t^2 \geq \frac{\eta}{4}\alpha_t.
  \end{align}
  With 
  \begin{align}
    z_t \doteq \sqrt{\E\left[\norm{u_{t} - w^*_{\theta_{t}}}^2\right]},
  \end{align}
  we then get 
  \begin{align}
    z_{t+1}^2 \leq \left(1 - \frac{\eta}{2}\alpha_t\right) z_t^2 + 2L_w L_\theta \alpha_t z_t + C_1 \alpha_t^2 \log^2 \alpha_t.
  \end{align}
  If 
  \begin{align}
    \left(1 - \frac{\eta}{2}\alpha_t\right) z_t^2 + 2L_w L_\theta \alpha_t z_t &\leq (1 - \frac{\eta}{3} \alpha_t) z_t^2, \\
    \label{eq special sa case1}
    \iff \frac{12 L_wL_\theta}{\eta} &\leq z_t,
  \end{align}
  we have
  \begin{align}
    z_{t+1}^2 \leq (1 - \frac{\eta}{3}\alpha_t) z_t^2 + C_1 \alpha_t^2 \log^2 \alpha_t.
  \end{align}
  If 
  \begin{align}
    \label{eq special sa case2}
    \frac{12L_wL_\theta}{\eta} &\geq z_t,
  \end{align}
  we have
  \begin{align}
    z_{t+1}^2 \leq& \left(1 - \frac{\eta}{2}\alpha_t\right) z_t^2 + 2 L_w L_\theta \alpha_t z_t + C_1 \alpha_t^2 \log^2 \alpha_t \\
    \leq &\left(1 - \frac{\eta}{2}\alpha_t\right) z_t^2 + \frac{24L_w^2L_\theta^2}{\eta} \alpha_t + C_1 \alpha_t^2 \log^2 \alpha_t.
  \end{align}
  Since for any time $t$, one of \eqref{eq special sa case1} and \eqref{eq special sa case2} must hold,
  we always have
  \begin{align}
    \label{eq special sa recursive bound}
    z_{t+1}^2 \leq &\left(1 - \frac{\eta}{3}\alpha_t\right) z_t^2 + \frac{24L_w^2L_\theta^2}{\eta} \alpha_t + C_1 \alpha_t^2 \log^2 \alpha_t.
  \end{align}
  Telescoping the above inequality from $t_0$ to $t$ yields
  \begin{align}
    z_{t}^2 \leq& \underbrace{\prod_{i=t_0}^{t-1}\left(1 - \frac{\eta}{3}\alpha_i\right)}_{E_1} z_{t_0}^2 + \frac{24 L_w^2L_\theta^2}{\eta} \underbrace{\sum_{i=t_0}^{t-1} \prod_{j=i+1}^{t-1} \left(1 - \frac{\eta}{3}\alpha_j\right) \alpha_i}_{E_2} \\
    &+C_1 \underbrace{\sum_{i=t_0}^{t-1} \prod_{j=i+1}^{t-1} \left(1 - \frac{\eta}{3}\alpha_j\right) \alpha_i^2 \log^2 \alpha_i}_{E_3},
  \end{align}
  where we adopt the convention that $\prod_{x=i}^j\left(\cdot\right) = 1$ if $i > j$.
  For $E_1$,
  using $1 + x \leq \exp x$ yields
  \begin{align}
    \label{eq special sa recursive bound1}
    E_1 \leq& \exp\left(-\frac{\eta}{3}\sum_{i=t_0}^{t-1} \alpha_i\right) \leq \exp\left(-\frac{\eta}{3} \int_{x=t_0}^t \frac{c_\alpha}{(t_0+x)^{\epsilon_\alpha}} dx\right) \\
    =& \begin{cases}
\left(\frac{2t_0}{t_0 + t}\right)^{\frac{\eta c_\alpha}{3}}, & \epsilon_\alpha = 1 \\
\exp\left( \frac{\eta c_\alpha}{3(1 - \epsilon_\alpha)} \left((2t_0)^{1 - \epsilon_\alpha} - (t_0 + t)^{1-\epsilon_\alpha}\right)\right), & \epsilon_\alpha \in (0, 1)
    \end{cases} .
  \end{align}
  For $E_2$,
  define
  \begin{align}
    B_t \doteq \sum_{i=0}^{t} \prod_{j=i+1}^{t} \left(1 - \frac{\eta}{3}\alpha_j\right) \alpha_i.
  \end{align}
  Then we have
  \begin{align}
    B_t = \alpha_t + (1 - \frac{\eta}{3}\alpha_t) B_{t-1}.
  \end{align}
  When $t_0$ is sufficiently large such that
  \begin{align}
    1 - \frac{\eta}{3} \alpha_t > 0,
  \end{align}
  it is easy to see
  \begin{align}
    B_{t-1} \leq \frac{3}{\eta} \implies B_t \leq \frac{3}{\eta}.
  \end{align}
  As $B_0 = \alpha_0$,
  we have $B_0 < \frac{3}{\eta}$ for sufficiently large $t_0$.
  We, therefore, conclude by induction that $\forall t$,
  \begin{align}
    B_t \leq \frac{3}{\eta}.
  \end{align}
  Consequently,
  \begin{align}
    \label{eq special sa recursive bound2}
    E_2 \leq B_{t-1} \leq \frac{3}{\eta}.
  \end{align}
  For $E_3$,
  we have
  \begin{align}
    \label{eq special sa recursive bound3}
    E_3 \leq& \log^2 \alpha_t \underbrace{\sum_{i=t_0}^{t-1} \prod_{j=i+1}^{t-1} \left(1 - \frac{\eta}{3}\alpha_j\right) \alpha_i^2}_{E_4}.
  \end{align}
  If $\epsilon_\alpha=1$,
  we have 
  \begin{align}
    E_4 \leq& \sum_{i=t_0}^{t-1} \exp\left(-\frac{\eta}{3} \int_{x=i+1}^t \frac{c_\alpha}{(t_0 + x)^{\epsilon_\alpha}} dx\right) \frac{c_\alpha^2}{(t_0 + i)^{2\epsilon_\alpha}} \\
    =& \sum_{i=t_0}^{t-1} \left(\frac{t_0 + i + 1}{t+t_0}\right)^{\frac{\eta c_\alpha}{3}} \frac{c_\alpha^2}{(t_0 + i)^2} \\
    =& \sum_{i=t_0}^{t-1} \left(\frac{t_0 + i + 1}{t+t_0}\right)^{\frac{\eta c_\alpha}{3}} \frac{c_\alpha^2}{(t_0 + i + 1)^2} \left(\frac{t_0+i+1}{t_0+i} \right)^2 \\
    \leq& \frac{4c_\alpha^2 }{(t+t_0)^{\frac{\eta c_\alpha}{3}}} \sum_{i=t_0}^{t-1} \frac{1}{(t_0+i+1)^{2 - \frac{\eta c_\alpha}{3}}} \\
    =& \begin{cases} 
      \fO\left(t^{-\frac{\eta c_\alpha}{3}} \right), & \eta c_\alpha \in (0, 3) \\
      \fO\left(\frac{\log t}{t}\right), & \eta c_\alpha = 3 \\
      \fO\left(\frac{1}{t}\right), & \eta c_\alpha \in (3, \infty)
    \end{cases}.
  \end{align}
  If $\epsilon_\alpha \in (0, 1)$,
  when $t_0$ is sufficiently large,
  we can use induction 
  (see, e.g., Section A.3.7 of \citet{chen2021lyapunov}) to show that
  \begin{align}
    E_4 = \fO\left(\frac{1}{t^{\epsilon_\alpha}}\right).
  \end{align}
  Putting the bounds in \eqref{eq special sa recursive bound1} , \eqref{eq special sa recursive bound2}, and \eqref{eq special sa recursive bound3} back into \eqref{eq special sa recursive bound} yields
  \begin{align}
    z_{t}^2 = \frac{72L_w^2 L_\theta^2}{\eta^2} + \begin{cases} 
      \fO\left(t^{-\frac{\eta c_\alpha}{3}} \log^2 t\right), &\epsilon_\alpha = 1, \eta c_\alpha \in (0, 3) \\
      \fO\left(\frac{\log^3 t}{t}\right), & \epsilon_\alpha = 1, \eta c_\alpha = 3 \\
      \fO\left(\frac{\log^2 t}{t}\right), & \epsilon_\alpha = 1, \eta c_\alpha \in (3, \infty) \\
      \fO\left(\frac{\log^2 t}{t^{\epsilon_\alpha}}\right), & \epsilon_\alpha \in (0, 1)
    \end{cases}.
  \end{align}
  Here we have used the fact that $E_3$ always dominates $E_1$ for any $\epsilon_\alpha, c_\alpha$.
  Using
  \begin{align}
    \E\left[\norm{w_t - w^*_{w_t}}^2\right] = \E\left[\norm{\Gamma(u_t) - \Gamma(w^*_{w_t})}^2\right] \leq \E\left[\norm{u_t - w^*_{w_t}}^2\right] = z_t^2
  \end{align}
  then completes the proof.
\end{proof}

\section{Proofs of Section~\ref{sec sarsa}}
\subsection{Proof of Theorem~\ref{thm sarsa convergence}}

\propcriticconvergence*
\label{sec proof thm sarsa convergence}
\begin{proof}
  To start with,
  define
  \begin{align}
    \label{eq definition of y}
    \fY &\doteq \qty{(s, a, s', a') \mid s \in \fS, a \in \fA, s' \in \fS, p(s'|s, a) > 0}, \\
    Y_t &\doteq (S_t, A_t, S_{t+1}, A_{t+1}), \\
    y &\doteq (s, a, s', a'), \\
    \tilde \pi_\theta(a|s) &\doteq \pi_{\Gamma(\theta)}(a|s), \\
    P_{\theta}((s_1, a_1, s_1', a_1'), (s_2, a_2, s_2', a_2')) &\doteq \begin{cases}
      0 & (s_1', a_1') \neq (s_2, a_2) \\
      p(s_2'|s_2, a_2) \tilde \pi_{\theta}(a_2'|s_2')  & (s_1', a_1') = (s_2, a_2)
    \end{cases}, \\
    F_{\theta}(w, y) \doteq& \left(r(s, a) + \gamma x(s', a')^\top w - x(s, a)^\top w\right)x(s, a) + w.
  \end{align}
  Here our $F_\theta(w, y)$ is actually independent of $\theta$.

  The update of $\qty{w_t}$ in Algorithm~\ref{alg sarsa lambda} with $\lambda = 0$ can then be expressed as
  \begin{align}
    w_{t+1} = \Gamma\left(w_t + \alpha_t \left(F_{\theta_t}(w_t, Y_t) - w_t\right)\right).
  \end{align}
  According to the action selection rule for $A_{t+1}$ specified in Algorithm~\ref{alg sarsa lambda},
  we have
  \begin{align}
    \Pr(Y_{t+1} = y) = P_{\theta_{t+1}}(Y_t, y),
  \end{align}
  Assumption~\ref{assu makovian} is then fulfilled.

  Assumption~\ref{assu uniform ergodicity} is immediately implied by Assumption~\ref{assu mu uniform ergodicity}.
  In particular, 
  for any $\theta$,
  the invariant distribution of 
  the chain induced by $P_\theta$ is $d_{\tilde \pi_{\theta}}(s)\tilde \pi_\theta(a|s)p(s'|s, a) \tilde \pi_\theta(a'|s')$.

  For Assumption~\ref{assu uniform contraction},
  it is easy to see
  \begin{align}
    \bar F_{\theta}(w) =& X^\top D_{\tilde \pi_\theta}(\gamma P_{\tilde \pi_\theta} - I)Xw + X^\top D_{\tilde \pi_\theta} r + w, \\
    f^{\alpha}_\theta(w) =& w + \alpha \left(X^\top D_{\tilde \pi_\theta}(\gamma P_{\tilde \pi_\theta} - I)Xw + X^\top D_{\tilde \pi_\theta} r\right).
  \end{align}
  Define 
  \begin{align}
    w^*_\theta \doteq -\left(X^\top D_{\tilde \pi_\theta}(\gamma P_{\tilde \pi_\theta} - I)X\right)^{-1} X^\top D_{\tilde \pi_\theta} r.
  \end{align}
  It is then easy to see that $w^*_\theta$ is the unique fixed point of $\bar F_\theta(w)$.
  The uniform pseudo-contraction is verified by Lemma~\ref{lem pseudo contraction}.
  In particular,
  we have
  \begin{align}
    \kappa_\alpha \doteq& \sqrt{1 - \eta \alpha}, \\
    \label{eq exact eta}
    \eta \doteq& (1-\gamma)\inf_{\theta} \lambda_{min}\left(X^\top D_{\pi_\theta} X\right) > 0,
  \end{align}
  where $\lambda_{min}(\cdot)$ denotes the minimum eigenvalue of a symmetric positive definite matrix.

  We now verify Assumption~\ref{assu regularization}.
  To verify Assumption~\ref{assu regularization} (i), 
  we have
  \begin{align}
    &\norm{F_\theta(w, y) - F_\theta(w', y)} \\
    \leq& \norm{\gamma x(s, a) x(s', a')^\top - x(s, a) x(s, a)^\top + I}\norm{w - w'} \\
    \leq& \underbrace{\left((1 + \gamma) x_{max}^2 + 1\right)}_{L_F} \norm{w - w'}.
  \end{align}

  Assumption~\ref{assu regularization} (ii) immediately holds since our $F_\theta(w, y)$ is independent of $\theta$.

  To verify Assumption~\ref{assu regularization} (iii),
  we have
  \begin{align}
    \norm{F_{\theta}(0, y)} = \norm{ r(s, a) x(s, a)} \leq \underbrace{r_{max} x_{max}}_{U_F}.
  \end{align}

  To verify Assumption~\ref{assu regularization} (iv),
  we have
  \begin{align}
    &\norm{\bar F_{\theta}(w) - \bar F_{\theta'}(w)}_\infty \\
    =& \norm{X^\top \left(D_{\tilde \pi_\theta}(\gamma P_{\tilde \pi_\theta} - I) - D_{\tilde \pi_{\theta'}}(\gamma P_{\tilde \pi_{\theta'}} - I) \right)Xw + X^\top \left(D_{\tilde \pi_\theta} - D_{\tilde \pi_{\theta'}}\right) r}_\infty \\
    \leq & \norm{X}_\infty^2\norm{D_{\tilde \pi_\theta}(\gamma P_{\tilde \pi_\theta} - I) - D_{\tilde \pi_{\theta'}}(\gamma P_{\tilde \pi_{\theta'}} - I)}_\infty \norm{w}_\infty + \norm{X}_\infty\norm{D_{\tilde \pi_\theta} - D_{\tilde \pi_{\theta'}}}_\infty \norm{r}_\infty.
  \end{align}
  Lemma~\ref{lem continuity of ergodic distribution} asserts that $D_{\pi_\theta}$ is Lipschitz continuous in $\theta$.
  We then conclude, by Lemma~\ref{lem product of lipschitz functions}, 
  that
  there exist positive constants $L_{DP} > 0, L_D > 0$ such that
  \begin{align}
    \norm{D_{\pi_\theta}(\gamma P_{\pi_\theta} - I) - D_{\pi_{\theta'}}(\gamma P_{\pi_{\theta'}} - I)}_\infty \leq& L_{DP} L_\pi \norm{\theta - \theta}_\infty, \\
    \norm{D_{\pi_\theta} - D_{\pi_{\theta'}}}_\infty\leq& L_D L_\pi \norm{\theta - \theta'}_\infty.
  \end{align}
  Importantly, $L_{DP}$ and $L_D$ do not depend on $C_{\Gamma}$.
  It is then easy to see that
  \begin{align}
    \norm{\bar F_{\theta}(w) - \bar F_{\theta'}(w)}_\infty \leq& \left(\norm{X}_\infty^2 L_{DP} L_\pi \norm{w}_\infty + \norm{X}_\infty \norm{r}_\infty L_D L_\pi  \right) \norm{\Gamma(\theta) - \Gamma(\theta')}_\infty \\
    \leq& \left(\norm{X}_\infty^2 L_{DP} L_\pi \norm{w} + \norm{X}_\infty \norm{r}_\infty L_D L_\pi  \right) \norm{\Gamma(\theta) - \Gamma(\theta')} \\
    \leq& \left(\norm{X}_\infty^2 L_{DP} L_\pi \norm{w} + \norm{X}_\infty \norm{r}_\infty L_D L_\pi \right) \norm{\theta - \theta'}.
  \end{align}
  It follow immediately that 
  \begin{align}
    \norm{\bar F_{\theta}(w) - \bar F_{\theta'}(w)} \leq& \sqrt{K}\left(\norm{X}_\infty^2 L_{DP} L_\pi \norm{w} + \norm{X}_\infty \norm{r}_\infty L_D L_\pi \right) \norm{\theta - \theta'}.
  \end{align}

  To verify Assumption~\ref{assu regularization} (v),
  we first use Lemma~\ref{lem bound of matrix inverse diff} to get
  \begin{align}
    &\norm{\left(X^\top D_{\tilde \pi_\theta}(\gamma P_{\tilde \pi_\theta} - I)X\right)^{-1} - \left(X^\top D_{\tilde \pi_{\theta'}}(\gamma P_{\tilde \pi_{\theta'}} - I)X\right)^{-1}}_\infty \\
    \leq&\norm{\left(X^\top D_{\tilde \pi_\theta}(\gamma P_{\tilde \pi_\theta} - I)X\right)^{-1}}_\infty \norm{\left(X^\top D_{\tilde \pi_{\theta'}}(\gamma P_{\tilde \pi_{\theta'}} - I)X\right)^{-1}}_\infty \\
    &\times \norm{X^\top D_{\tilde \pi_\theta}(\gamma P_{\tilde \pi_\theta} - I)X - X^\top D_{\tilde \pi_{\theta'}}(\gamma P_{\tilde \pi_{\theta'}} - I)X}_\infty.
  \end{align}
  Thanks to Assumption~\ref{assu mu uniform ergodicity},
  for any $\theta$,
  \begin{align}
    \left(X^\top D_{\pi_\theta}(\gamma P_{\pi_\theta} - I)X\right)^{-1}
  \end{align}
  is well-defined.
  Since $\bar \Lambda_\pi$ is a compact set,
  we conclude, by the extreme value theorem,
  that there exists a constant $U_{inv} > 0$,
  independent of $C_\Gamma$,
  such that
  \begin{align}
    \sup_\theta \norm{\left(X^\top D_{\pi_\theta}(\gamma P_{ \pi_\theta} - I)X\right)^{-1}}_\infty < U_{inv}.
  \end{align}
  Recalling that $\tilde \pi_\theta(a|s) = \pi_{\Gamma(\theta)}(a|s)$ then yields
  \begin{align}
    \sup_\theta \norm{\left(X^\top D_{\tilde \pi_\theta}(\gamma P_{\tilde \pi_\theta} - I)X\right)^{-1}}_\infty < U_{inv}.
  \end{align}
  It then follows immediately that
  \begin{align}
    &\norm{\left(X^\top D_{\tilde \pi_\theta}(\gamma P_{\tilde \pi_\theta} - I)X\right)^{-1} - \left(X^\top D_{\tilde \pi_{\theta'}}(\gamma P_{\tilde \pi_{\theta'}} - I)X\right)^{-1}}_\infty \leq U_{inv}^2 \norm{X}_\infty^2 L_{DP} L_\pi \norm{\theta - \theta'}_\infty.
  \end{align}
  It is also easy to see that
  \begin{align}
    \norm{X^\top D_{\tilde \pi_\theta} r}_\infty &\leq \norm{X}_\infty \norm{r}_\infty, \\
    \norm{X^\top D_{\tilde \pi_\theta} r - X^\top D_{\tilde \pi_{\theta'}} r}_\infty &\leq \norm{X}_\infty L_D L_\pi \norm{r}_\infty.
  \end{align}
  Using Lemma~\ref{lem product of lipschitz functions} again yields
  \begin{align}
    \norm{w^*_\theta - w^*_{\theta'}}_\infty \leq \left(U_{inv}^2 \norm{X}_\infty^2 L_{DP} \norm{X}_\infty + U_{inv} \norm{X}_\infty L_D\right) L_\pi \norm{r}_\infty \norm{\theta - \theta'}_\infty.
  \end{align}
  It follows immediately that
  \begin{align}
    \label{eq exact lw}
    \norm{w^*_\theta - w^*_{\theta'}} \leq \underbrace{\sqrt{K}\left(U_{inv}^2 \norm{X}_\infty^2 L_{DP} \norm{X}_\infty + U_{inv} \norm{X}_\infty L_D\right) L_\pi \norm{r}_\infty}_{L_w} \norm{\theta - \theta'}.
  \end{align}

  To verify Assumption~\ref{assu regularization} (vi),
  we have
  \begin{align}
    \sup_\theta \norm{w^*_\theta}_\infty \leq U_{inv}\norm{X}_\infty \norm{r}_\infty.
  \end{align}
  It follows immediately that 
  \begin{align}
    \label{eq exact uw}
    \sup_\theta \norm{w^*_\theta} \leq \underbrace{\sqrt{K} U_{inv}\norm{X}_\infty \norm{r}_\infty}_{U_w}.
  \end{align}

  Assumption~\ref{assu regularization} (vii) follows immediately from Assumption~\ref{assu lipschitz mu}.

  We now verify Assumption~\ref{assu projection}.
  Assumption~\ref{assu projection} (i) is fulfilled by our selection of $C_\Gamma$.
  It is easy to see 
  \begin{align}
    \tilde \pi_{\theta}(a|s) = \tilde \pi_{\Gamma(\theta)} (a|s),
  \end{align}
  Assumption~\ref{assu projection} (ii) then follows immediately.

  With Assumptions~\ref{assu makovian} -~\ref{assu projection} satisfied,
  we conclude by Corollary~\ref{thm sa single weight} that
  the iterates $\qty{w_t}$ generated by Algorithm~\ref{alg sarsa lambda} with $\lambda = 0$ satisfy
  \begin{align}
  \E\left[\norm{w_{t} - w^*_{w_{t}}}^2\right] = \frac{72L_w^2 L_\theta^2}{\eta^2} + \begin{cases} 
    \fO\left(t^{-\frac{\eta c_\alpha}{3}} \log^2 t\right), &\epsilon_\alpha = 1, \eta c_\alpha \in (0, 3) \\
    \fO\left(\frac{\log^3 t}{t}\right), & \epsilon_\alpha = 1, \eta c_\alpha = 3 \\
    \fO\left(\frac{\log^2 t}{t}\right), & \epsilon_\alpha = 1, \eta c_\alpha \in (3, \infty) \\
    \fO\left(\frac{\log^2 t}{t^{\epsilon_\alpha}}\right), & \epsilon_\alpha \in (0, 1)
  \end{cases}.
  \end{align}
  where 
  \begin{align}
    L_\theta \doteq& U_F + (L_F + 1) C_\Gamma = \left(r_{max}x_{max} + \left((1+\gamma)x_{max}^2 + 2\right) C_{\Gamma} \right) \\
    \leq& 1 + 4 C_\Gamma.
  \end{align}
  Consequently,
  \begin{align}
    \E\left[\norm{w_t - w^*_{w_t}}\right] \leq& \sqrt{\E\left[\norm{w_{t} - w^*_{w_{t}}}^2\right]} \\
    =& \frac{6\sqrt{2}L_w L_\theta}{\eta} + \begin{cases} 
    \fO\left(t^{-\frac{\eta c_\alpha}{6}} \log t\right), & \epsilon_\alpha = 1, \eta c_\alpha \in (0, 3) \\
    \fO\left(t^{-\frac{1}{2}} \log^{\frac{3}{2}}t \right), & \epsilon_\alpha = 1, \eta c_\alpha = 3 \\
    \fO\left(t^{-\frac{1}{2}} \log t\right), & \epsilon_\alpha = 1, \eta c_\alpha \in (3, \infty) \\
    \fO\left(t^{-\frac{\epsilon_\alpha}{2}} \log t\right), &\epsilon_\alpha \in (0, 1) 
  \end{cases}.
  \end{align}
  If 
  \begin{align}
    \norm{r}_\infty < \frac{1}{\sqrt{K}\left(U_{inv}^2 \norm{X}_\infty^2 L_{DP} \norm{X}_\infty + U_{inv} \norm{X}_\infty L_D\right)},
  \end{align}
  we get
  \begin{align}
    L_w < 1.
  \end{align}
  Since
  \begin{align}
    &\E\left[\norm{w_t - w_*}\right] \\
    =&\E\left[\norm{w_t - w^*_{w_*}}\right] \\
    \leq& \E\left[\norm{w_t - w^*_{w_t}}\right] + \E\left[\norm{w^*_{w_t} - w^*_{w_*}}\right] \\
    \leq& \E\left[\norm{w_t - w^*_{w_t}}\right] + L_w \E\left[\norm{w_t - w_*}\right],
  \end{align}
  we conclude that
  \begin{align}
    \E\left[\norm{w_t - w_*}\right] \leq& \frac{1}{1 - L_w} \E\left[\norm{w_t - w^*_{w_t}}\right]\\
    =& \frac{6\sqrt{2}L_w \left(1+4C_\Gamma\right)}{\eta (1-L_w)} + 
    \begin{cases} 
    \fO\left(t^{-\frac{\eta c_\alpha}{6}} \log t\right), & \epsilon_\alpha = 1, \eta c_\alpha \in (0, 3) \\
    \fO\left(t^{-\frac{1}{2}} \log^{\frac{3}{2}}t \right), & \epsilon_\alpha = 1, \eta c_\alpha = 3 \\
    \fO\left(t^{-\frac{1}{2}} \log t\right), & \epsilon_\alpha = 1, \eta c_\alpha \in (3, \infty) \\
    \fO\left(t^{-\frac{\epsilon_\alpha}{2}} \log t\right), &\epsilon_\alpha \in (0, 1) 
    \end{cases},
  \end{align}
  which completes the proof.

\end{proof}

\section{Technical Lemmas}
\begin{lemma}
  \label{lem product of lipschitz functions}
  Let $f_1(x), f_2(x)$ be two Lipschitz continuous functions with Lipschitz constants $L_1, L_2$.
  Assume $\norm{f_1(x)} \leq U_1, \norm{f_2(x)} \leq U_2$,
  then
  $L_1U_2 + L_2U_1$ is a Lipschitz constant of $f(x) \doteq f_1(x)f_2(x)$. 
\end{lemma}
\begin{proof}
  \begin{align}
    &\norm{f_1(x)f_2(x) - f_1(y)f_2(y)} \\
    \leq& \norm{f_1(x)} \norm{f_2(x) - f_2(y)} + \norm{f_2(y)}\norm{f_1(x) - f_1(y)} \\
    \leq& (U_1 L_2 + U_2L_1) \norm{x-y}.
  \end{align}
\end{proof}

\begin{lemma}
  \label{lem bound of xk diff}
  Given positive integers $t_1 < t_2$ satisfying
  \begin{align}
      \alpha_{t_1, t_2 - 1} \leq \frac{1}{4A},
  \end{align}
  we have, for any $t \in [t_1, t_2]$,
  \begin{align}
      \label{eq bound of xk diff1}
      \norm{\Gamma(u_t) - \Gamma(u_{t_1})} &\leq 2 \alpha_{t_1, t_2 - 1}(A\norm{\Gamma(u_{t_1})} + B), \\
      \label{eq bound of xk diff2}
      \norm{\Gamma(u_t) - \Gamma(u_{t_1})} &\leq 4 \alpha_{t_1, t_2 - 1}(A \norm{\Gamma(u_{t_2})} + B), \\
      \label{eq bound of xk diff3}
      \norm{\Gamma(u_t) - \Gamma(u_{t_1})} &\leq \min \qty{\norm{\Gamma(u_{t_1})}, \norm{\Gamma(u_{t_2})} } + \frac{B}{A}.
  \end{align}
\end{lemma}
\begin{proof}
  Notice that 
  \begin{align}
      &\norm{\Gamma(u_{t+1})} - \norm{\Gamma(u_t)} \\
      \leq& \norm{\Gamma(u_{t+1}) - \Gamma(u_t)} \\
      =& \norm{\Gamma(u_{t+1}) - \Gamma(\Gamma(u_t))} \\
      \leq& \norm{u_{t+1} - \Gamma(u_t)} \\
      =& \alpha_t \norm{F_{\theta_t}(\Gamma(u_t), Y_t) - \Gamma(u_t)}  \\
      \leq& \alpha_t \left( \norm{F_{\theta_t}(\Gamma(u_t), Y_t)} + \norm{\Gamma(u_t)}\right) \\
      \leq&\alpha_t (U_F + (L_F + 1)\norm{\Gamma(u_t)} \qq{(Lemma~\ref{lem bound of fxy})} \\
      \label{eq tmp 8}
      \leq& \alpha_t (A\norm{\Gamma(u_t)} + B) \qq{(Using \eqref{eq shorthand a and b})}
  \end{align}
  The rest of the proof follows from the proof of Lemma A.2 of \citet{chen2021lyapunov} up to changes of notations. 
  We include it for completeness.
  Rearranging terms of the above inequality yields 
  \begin{align}
    \norm{\Gamma(u_{t+1})} + \frac{B}{A} \leq (1 + \alpha_t A)\left(\norm{\Gamma(u_t)} + \frac{B}{A}\right),
  \end{align}
  implying that for any $t \in (t_1, t_2]$,
  \begin{align}
    \norm{\Gamma(u_t)} + \frac{B}{A} \leq \prod_{j=t_1}^{t-1} (1 + A\alpha_j) \left(\norm{\Gamma(u_{t_1})} + \frac{B}{A}\right) .
  \end{align}
  Notice that for any $x \in [0, \frac{1}{2}]$,
  $1 + x \leq \exp(x) \leq 1 + 2x$ always hold.
  Hence 
  \begin{align}
    \alpha_{t_1, t_2-1} \leq \frac{1}{4A} 
  \end{align}
  implies
  \begin{align}
    \prod_{j=t_1}^{t-1} (1 + A\alpha_j) \leq \exp(A\alpha_{t_1, t-1}) \leq 1 + 2A \alpha_{t_1, t-1}.
  \end{align}
  Consequently, for any $t \in (t_1, t_2]$,
  we have
  \begin{align}
    \norm{\Gamma(u_t)} + \frac{B}{A} &\leq \left(1 + 2A \alpha_{t_1, t-1}\right) \left(\norm{\Gamma(u_{t_1})} + \frac{B}{A} \right) \\
    \implies \norm{\Gamma(u_t)} & \leq \left(1 + 2A \alpha_{t_1, t-1}\right) \norm{\Gamma(u_{t_1})} + 2B\alpha_{t_1, t-1},
  \end{align}
  which together with \eqref{eq tmp 8} yields that for any $t \in (t_1, t_2 - 1]$
  \begin{align}
    \norm{\Gamma(u_{t+1}) - \Gamma(u_t)} &\leq \alpha_t \left(A \norm{\Gamma(u_t)} + B\right) \\
    &\leq \alpha_t \left(A \left(1 + 2A \alpha_{t_1, t-1}\right) \norm{\Gamma(u_{t_1})} + 2AB\alpha_{t_1, t-1} + B\right) \\
    &\leq 2 \alpha_t(A\norm{\Gamma(u_{t_1})} + B) \qq{(Using $\alpha_{t_1, t-1} \leq \frac{1}{4A}$)}.
  \end{align}
  Consequently, for any $t \in (t_1, t_2]$, we have
  \begin{align}
    \norm{\Gamma(u_t) - \Gamma(u_{t_1})} &\leq \sum_{j=t_1}^{t-1} \norm{\Gamma(w_{j+1}) - \Gamma(w_j)} \leq \sum_{j=t_1}^{t-1} 2 \alpha_j (A \norm{\Gamma(u_{t_1})} + B) \\
    &= 2 \alpha_{t_1, t-1}(A \norm{\Gamma(u_{t_1})} + B) \leq 2 \alpha_{t_1, t_2-1} (A\norm{\Gamma(u_{t_1})} +B),
  \end{align}
  which completes the proof of \eqref{eq bound of xk diff1}. 
  For \eqref{eq bound of xk diff2},
  we have from the above inequality
  \begin{align}
    \norm{\Gamma(u_{t_2}) - \Gamma(u_{t_1})} \leq& 2\alpha_{t_1, t_2-1} (A\norm{\Gamma(u_{t_1})} + B) \\
    \leq &2\alpha_{t_1, t_2-1} (A\norm{\Gamma(u_{t_1}) - \Gamma(u_{t_2})} + A\norm{\Gamma(u_{t_2})} + B) \\
    \leq& \frac{1}{2}\norm{\Gamma(u_{t_1}) - \Gamma(u_{t_2})} + 2 \alpha_{t_1, t_2 -1} (A\norm{\Gamma(u_{t_2})} + B),
  \end{align}
  implying
  \begin{align}
    \norm{\Gamma(u_{t_2}) - \Gamma(u_{t_1})} \leq 4\alpha_{t_1, t_2 -1} (A\norm{\Gamma(u_{t_2})} + B).
  \end{align}
  Consequently, for any $t \in [t_1, t_2]$,
  \begin{align}
    \norm{\Gamma(u_t) - \Gamma(u_{t_1})}\leq& 2\alpha_{t_1, t_2-1} (A\norm{\Gamma(u_{t_1})} + B) \\
    \leq &2\alpha_{t_1, t_2-1} (A\norm{\Gamma(u_{t_1}) - \Gamma(u_{t_2})} + A\norm{\Gamma(u_{t_2})} + B) \\
    \leq&2\alpha_{t_1, t_2-1} \left(A4\alpha_{t_1, t_2 -1} (A\norm{\Gamma(u_{t_2})} + B) + A\norm{\Gamma(u_{t_2})} + B\right) \\
    \leq& 4\alpha_{t_1, t_2-1}(A\norm{\Gamma(u_{t_2})} + B) \qq{(Using $\alpha_{t_1, t_2-1} \leq \frac{1}{4A}$)},
  \end{align}
  which completes the proof of \eqref{eq bound of xk diff2}.
  \eqref{eq bound of xk diff1} implies
  \begin{align}
    \norm{\Gamma(u_t) - \Gamma(u_{t_1})} \leq \norm{\Gamma(u_{t_1})} + \frac{B}{A},
  \end{align}
  \eqref{eq bound of xk diff2} implies
  \begin{align}
    \norm{\Gamma(u_t) - \Gamma(u_{t_1})} \leq \norm{\Gamma(u_{t_2})} + \frac{B}{A},
  \end{align}
  then \eqref{eq bound of xk diff3} follows immediately,
  which completes the proof.
\end{proof}

\begin{lemma}
  \label{lem continuity of ergodic distribution}
  Let Assumptions \ref{assu lipschitz mu} and \ref{assu mu uniform ergodicity} hold.
  Then there exists a constant $L_{\pi}'$ such that $\forall \theta, \theta', a, s$,
  \begin{align}
    \abs{d_{\pi_\theta}(s, a) - d_{\pi_{\theta'}}(s, a)} \leq L_\pi' \norm{\theta - \theta'}_\infty.
  \end{align}
\end{lemma}
\begin{proof}
  See, e.g., Lemma 9 of \citet{zhang2021breaking}.
\end{proof}

\begin{lemma}
  \label{lem bound of matrix inverse diff}
  For any $\norm{\cdot}$,
  we have
  \begin{align}
    \norm{X^{-1} - Y^{-1}} \leq \norm{X^{-1}} \norm{X-Y} \norm{Y^{-1}}.
  \end{align}
\end{lemma}
\begin{proof}
  \begin{align}
    \norm{X^{-1} - Y^{-1}} &= \norm{X^{-1}YY^{-1} - X^{-1}XY^{-1}} \leq \norm{X^{-1}} \norm{X-Y} \norm{Y^{-1}}.
  \end{align}
\end{proof}

\begin{lemma}
  \label{lem accu learning rates}
  Recall that
  \begin{align}
    L_{\theta} = U_F + (L_F + 1)C_\Gamma,
  \end{align}
  then
  for any $j > i, y, y', w$,
  \begin{align}
    \norm{w^*_{\theta_j} - w^*_{\theta_i}} \leq& L_w L_\theta \alpha_{i, j-1}, \\
    \abs{P_{\theta_j}(y, y') - P_{\theta_i}(y, y')} \leq& L_P L_\theta \alpha_{i, j-1}, \\
    \norm{F_{\theta_j}(w, y) - F_{\theta_i}(w, y)} \leq& L_F'L_\theta \alpha_{i, j-1}\left(\norm{w} + U_F'\right), \\
    \norm{\bar F_{\theta_j}(w) - \bar F_{\theta_i}(w)} \leq& L_F''L_\theta \alpha_{i, j-1}\left(\norm{w} + U_F''\right).
  \end{align}
\end{lemma}
\begin{proof}
  \begin{align}
    \norm{w^*_{\theta_j} - w^*_{\theta_i}} \leq& \sum_{k=i}^{j-1} \norm{w^*_{\theta_{k+1}} - w^*_{\theta_k}} \\
    \leq& \sum_{k=i}^{j-1} \norm{w^*_{\theta_{k+1}} - w^*_{\Gamma(\theta_k)}} \qq{(Assumption~\ref{assu projection})}\\
    \leq& \sum_{k=i}^{j-1} L_w \norm{\theta_{k+1} - \Gamma(\theta_k)} \\
    =& \sum_{k=i}^{j-1} L_w \norm{w_{k+1} - \Gamma(w_k)} \\
    =& \sum_{k=i}^{j-1} L_w \norm{\Gamma(u_{k+1}) - \Gamma(\Gamma(u_k))} \qq{(Lemma~\ref{lem transform})} \\
    \leq & \sum_{k=i}^{j-1} L_w \norm{u_{k+1} - \Gamma(u_k)} \\
    =& \sum_{k=i}^{j-1} L_w \alpha_k \norm{F_{\theta_k}(\Gamma(u_k), Y_k) - \Gamma(u_k)} \\
    \leq& \sum_{k=i}^{j-1} L_w \alpha_k \left(U_F + L_F \norm{\Gamma(u_k)} + \norm{\Gamma(u_k)} \right) \qq{(Lemma~\ref{lem bound of fxy})} \\
    \leq& \sum_{k=i}^{j-1} L_w \alpha_k \left(U_F + L_F C_{\Gamma} + C_{\Gamma} \right) \\
    =& L_w L_\theta \alpha_{i, j-1}.
  \end{align}
  Similarly we can get
  \begin{align}
    \abs{P_{\theta_j}(y, y') - P_{\theta_i}(y, y')} \leq&  L_P L_\theta \alpha_{i, j-1}.
  \end{align}
  Moreover,
  \begin{align}
    \norm{F_{\theta_j}(w, y) - F_{\theta_i}(w, y)} \leq& \sum_{k=i}^{j-1} \norm{F_{\theta_{k+1}}(w, y) - F_{\theta_k}(w, y)} \\
    \leq& \sum_{k=i}^{j-1} \norm{F_{\theta_{k+1}}(w, y) - F_{\Gamma(\theta_k)}(w, y)} \qq{(Assumption~\ref{assu projection})} \\
    \leq& \sum_{k=i}^{j-1} L_F' \norm{\theta_{k+1} - \Gamma(\theta_k)}\left(\norm{w} + U_F'\right) \\
    \leq& L_F'L_\theta \alpha_{i, j-1}\left(\norm{w} + U_F'\right).
  \end{align}
  Since $P_\theta = P_{\Gamma(\theta)}$,
  it is easy to see $d_{\theta}(y) = d_{\Gamma(\theta)}(y)$.
  Consequently, $\bar F_{\theta}(w) = \bar F_{\Gamma(\theta)}(w)$.
  We can then similarly get 
  \begin{align}
    \norm{\bar F_{\theta_j}(w) - \bar F_{\theta_i}(w)} \leq& L_F''L_\theta \alpha_{i, j-1}\left(\norm{w} + U_F''\right),
  \end{align}
  which completes the proof.
\end{proof}

\begin{lemma}
  \label{lem pseudo contraction}
  (Lemma 5.4 of \citet{de2000existence})
  There exists an $\bar \alpha$ such that for all $\alpha \in (0, \bar \alpha)$ and all $\theta$,
  \begin{align}
    \norm{f^\alpha_\theta(w) - w^*_\theta} \leq \kappa_\alpha \norm{w - w^*_\theta},
  \end{align}
  where
  \begin{align}
    \kappa_\alpha \doteq \sqrt{1 - (1-\gamma)\inf_{\theta} \lambda_{min}\left(X^\top D_{\pi_\theta} X\right) \alpha} < 1.
  \end{align}
  Here $\lambda_{min}(\cdot)$ denotes the minimum eigenvalue of a symmetric positive definite matrix.
\end{lemma}
\begin{proof}
  The proof is due to \citet{de2000existence}; we rewrite it in our notation for completeness.
  We first recall
  \begin{align}
    f^{\alpha}_\theta(w) =& w + \alpha \left(X^\top D_{\tilde \pi_\theta}(\gamma P_{\tilde \pi_\theta} - I)Xw + X^\top D_{\tilde \pi_\theta} r\right), \\
    w^*_\theta =& -\left(X^\top D_{\tilde \pi_\theta}(\gamma P_{\tilde \pi_\theta} - I)X\right)^{-1} X^\top D_{\tilde \pi_\theta} r.
  \end{align}
  Define 
  \begin{align}
    g_\theta(w) \doteq& X^\top D_{\tilde \pi_\theta}(\gamma P_{\tilde \pi_\theta} - I)Xw + X^\top D_{\tilde \pi_\theta} r \\
    =& X^\top D_{\tilde \pi_\theta} X \left(X^\top D_{\tilde \pi_\theta} X\right)^{-1}X^\top D_{\tilde \pi_\theta} \left(\bop_{\tilde \pi_\theta} Xw - Xw\right) \\
    =& X^\top D_{\tilde \pi_\theta} \Pi_{D_{\tilde \pi_\theta}} \bop_{\tilde \pi_\theta} Xw - X^\top D_{\tilde \pi_\theta} Xw \\
    =& X^\top D_{\tilde \pi_\theta} \left(\Pi_{D_{\tilde \pi_\theta}} \bop_{\tilde \pi_\theta} Xw - Xw \right).
  \end{align}
  By the contraction property (see, e.g., \citet{tsitsiklis1997analysis}),
  \begin{align}
    \norm{\Pi_{D_{\tilde \pi_\theta}} \bop_{\tilde \pi_\theta} Xw - Xw^*_\theta}_{D_{\tilde \pi_\theta}} \leq \gamma \norm{Xw - Xw^*_\theta}_{D_{\tilde \pi_\theta}}.
  \end{align}
  Consequently,
  \begin{align}
    &(w - w^*_\theta)^\top g_\theta(s) \\
    =& \left(Xw - Xw^*_\theta\right)^\top D_{\tilde \pi_\theta} \left(\Pi_{D_{\tilde \pi_\theta}} \bop_{\tilde \pi_\theta} Xw - Xw \right) \\
    =& \left(Xw - Xw^*_\theta\right)^\top D_{\tilde \pi_\theta} \left(\Pi_{D_{\tilde \pi_\theta}} \bop_{\tilde \pi_\theta} Xw - Xw^*_\theta + Xw^*_\theta - Xw \right) \\
    \leq& \norm{Xw - Xw^*_\theta}_{D_{\tilde \pi_\theta}} \norm{\Pi_{D_{\tilde \pi_\theta}} \bop_{\tilde \pi_\theta} Xw - Xw^*_\theta}_{D_{\tilde \pi_\theta}} - \norm{Xw - Xw^*_\theta}_{D_{\tilde \pi_\theta}}^2 
    \intertext{\hfill (Cauthy-Schwarz inequality)}
    \leq& (\gamma - 1)\norm{Xw - Xw^*_\theta}_{D_{\tilde \pi_\theta}}^2 \\
    =& (\gamma - 1) (w - w^*_\theta)^\top \left(X^\top D_{\tilde\pi_\theta} X\right)(w - w^*_\theta).
  \end{align}
  Since $X^\top D_{\pi_\theta} X$ is symmetric and positive define,
  eigenvalues are continuous in the elements of the matrix,
  $\bar \Lambda_\pi$ is compact,
  we conclude,
  by the extreme value theorem,
  that 
  \begin{align}
    C_1 \doteq \inf_{\theta} \lambda_{min}\left(X^\top D_{\pi_\theta} X\right) > 0.
  \end{align}
  Consequently, 
  for any $y$ and $\theta$,
  \begin{align}
    y^\top X^\top D_{\pi_\theta} X y \geq C_1 \norm{y}^2,
  \end{align}
  implying
  \begin{align}
    y^\top X^\top D_{\tilde \pi_\theta} X y \geq C_1 \norm{y}^2.
  \end{align}
  It follows immediately that
  \begin{align}
    \label{eq tmp 9}
    (w - w^*_\theta)^\top g_\theta(w) \leq -(1-\gamma) C_1 \norm{w - w^*_\theta}^2.
  \end{align}
  Moreover, 
  let $x_i$ be the $i$-the column of $X$,
  we have
  \begin{align}
    \norm{g_\theta(w)}^2 =& \sum_{i=1}^K \left(x_i^\top D_{\tilde \pi_\theta} \left(\Pi_{D_{\tilde \pi_\theta}} \bop_{\tilde \pi_\theta} Xw - Xw \right) \right)^2 \\
    \leq& \sum_{i=1}^K \norm{x_i}^2_{D_{\tilde \pi_\theta}}  \norm{\Pi_{D_{\tilde \pi_\theta}} \bop_{\tilde \pi_\theta} Xw - Xw}^2_{D_{\tilde \pi_\theta}} \\
    \intertext{\hfill (Cauchy-Schwarz inequality)}
    \leq& \sum_{i=1}^K \norm{x_i}^2_{D_{\tilde \pi_\theta}}  \left(\norm{\Pi_{D_{\tilde \pi_\theta}} \bop_{\tilde \pi_\theta} Xw -Xw^*_\theta}_{D_{\tilde \pi_\theta}} + \norm{Xw^*_\theta - Xw}_{D_{\tilde \pi_\theta}}\right)^2 \\
    \leq& (1+\gamma)^2 \sum_{i=1}^K \norm{x_i}^2_{D_{\tilde \pi_\theta}}  \norm{Xw^*_\theta - Xw}^2_{D_{\tilde \pi_\theta}} \\
    =& (1+\gamma)^2 \left(\sum_{i=1}^K \norm{x_i}^2_{D_{\tilde \pi_\theta}} \right) \norm{X^\top D_{\tilde \pi_\theta} X} \norm{w - w^*_\theta}^2.
  \end{align}
  According to the extreme value theorem,
  \begin{align}
    C_2 \doteq \sup_\theta \left(\sum_{i=1}^K \norm{x_i}^2_{D_{\pi_\theta}} \right) \norm{X^\top D_{\pi_\theta} X} < \infty.
  \end{align}
  Consequently,
  we have
  \begin{align}
    \label{eq tmp 10}
    \norm{g_\theta(w)}^2 \leq (1+\gamma)^2C_2 \norm{w - w^*_\theta}^2.
  \end{align}
  Combining~\eqref{eq tmp 9} and~\eqref{eq tmp 10} yields
  \begin{align}
    \norm{f_\theta^\alpha(w) - w^*_\theta}^2 =& \norm{w + \alpha g_\theta(w) - w^*_\theta}^2 \\
    =& \norm{w - w^*_\theta}^2 + 2 \alpha (w - w^*_\theta)^\top g_\theta(w) + \alpha^2 \norm{g_\theta(w)}^2 \\
    \leq& \left(1 - 2\alpha(1 - \gamma) C_1 + (1+\gamma)^2 \alpha^2 C_2 \right) \norm{w - w^*_\theta}^2.
  \end{align}
  Consequently,
  if 
  \begin{align}
    \alpha < \bar \alpha \doteq \frac{(1 - \gamma)C_1}{(1 + \gamma)^2 C_2},
  \end{align}
  we have
  \begin{align}
    1 - 2\alpha(1 - \gamma) C_1 + (1+\gamma)^2 \alpha^2 C_2 \leq 1 - (1-\gamma)C_1\alpha.
  \end{align}
  Defining 
  \begin{align}
    \kappa_\alpha \doteq \sqrt{1 - (1-\gamma)C_1\alpha}
  \end{align}
  then completes the proof.
  Importantly,
  both $C_1$ and $C_2$ here are independent of $C_\Gamma$.
\end{proof}

\section{Proof of Auxiliary Lemmas}

\subsection{Proof of Lemma~\ref{lem bound t1}}
\label{sec proof lem bound t1}
\lemboundtone*
\begin{proof}
  \begin{align}
      T_1 =&\indot{\Gamma(u_t) - w^*_{\theta_t}}{w^*_{\theta_t} - w^*_{\theta_{t+1}}} \\
      \leq & \norm{\Gamma(u_t) - w^*_{\theta_t}} \norm{w^*_{\theta_t} - w^*_{\theta_{t+1}}} \\
      \leq & \norm{\Gamma(u_t) - w^*_{\theta_t}} L_w L_\theta \alpha_t \qq{(Lemma~\ref{lem accu learning rates})}
  \end{align}
\end{proof}

\subsection{Proof of Lemma~\ref{lem bound t2}}
\label{sec proof lem bound t2}
\lemboundttwo*
\begin{proof}
\begin{align}
  T_2 = & \indot{\Gamma(u_t) - w^*_{\theta_t}}{f^{\alpha_t}_{\theta_t}(\Gamma(u_t)) - \Gamma(u_t)} \\
  =& \indot{\Gamma(u_t) - w^*_{\theta_t}}{f^{\alpha_t}_{\theta_t}(\Gamma(u_t)) - w^*_{\theta_t}} - \indot{\Gamma(u_t) - w^*_{\theta_t}}{\Gamma(u_t) - w^*_{\theta_t}} \\
  \leq & \norm{\Gamma(u_t) - w^*_{\theta_t}} \norm{f^{\alpha_t}_{\theta_t}(\Gamma(u_t)) - w^*_{\theta_t}} - \norm{\Gamma(u_t) - w^*_{\theta_t}}^2 \\
  \leq & \norm{\Gamma(u_t) - w^*_{\theta_t}} \kappa_{\alpha_t} \norm{\Gamma(u_t) - w^*_{\theta_t}} - \norm{\Gamma(u_t) - w^*_{\theta_t}}^2 \qq{(Assumption~\ref{assu uniform contraction})} \\
  = & -(1 - \kappa_{\alpha_t}) \norm{\Gamma(u_t) - w^*_{\theta_t}}^2.
\end{align}
\end{proof}

\subsection{Proof of Lemma~\ref{lem bound of t31}}
\label{sec proof lem bound t31}
\lemboundoftthreeone*
\begin{proof}
  \begin{align}
      T_{31} = &\indot{\Gamma(u_t) - w^*_{\theta_t} - \left(\Gamma(u_{t-\tau_{\alpha_t}}) - w^*_{\theta_{t-\tau_{\alpha_t}}}\right)}{F_{\theta_t}(\Gamma(u_t), Y_t) - \bar F_{\theta_t}(\Gamma(u_t))} \\
      \leq &\norm{\Gamma(u_t) - w^*_{\theta_t} - \left(\Gamma(u_{t-\tau_{\alpha_t}}) - w^*_{\theta_{t-\tau_{\alpha_t}}}\right)}\norm{F_{\theta_t}(\Gamma(u_t), Y_t) - \bar F_{\theta_t}(\Gamma(u_t))}.
  \end{align}
For the first term,
we have
\begin{align}
  &\norm{\Gamma(u_t) - w^*_{\theta_t} - \left(\Gamma(u_{t-\tau_{\alpha_t}}) - w^*_{\theta_{t-\tau_{\alpha_t}}}\right)} \\
  \leq& \norm{\Gamma(u_t) - \Gamma(u_{t-\tau_{\alpha_t}})} + \norm{w^*_{\theta_t} - w_{\theta_{t-\tau_{\alpha_t}}}^*} \\
  \leq& \norm{\Gamma(u_t) - \Gamma(u_{t-\tau_{\alpha_t}})} + L_w L_\theta \alpha_{t-\tau_{\alpha_t}, t-1} \qq{(Lemma~\ref{lem accu learning rates})}\\
  \leq& 4 \alpha_{t-\tau_{\alpha_t}, t-1} (A\norm{\Gamma(u_t)} + B)+  L_w L_\theta \alpha_{t-\tau_{\alpha_t}, t-1} \qq{(Lemma \ref{lem bound of xk diff})}\\
  \leq& 4 \alpha_{t-\tau_{\alpha_t}, t-1} (A\norm{\Gamma(u_t) - w^*_{\theta_t}} + A\norm{w^*_{\theta_t}} + B)+  L_w L_\theta \alpha_{t-\tau_{\alpha_t}, t-1} \\
  \leq& 4 \alpha_{t-\tau_{\alpha_t}, t-1} (L_wL_\theta + 1) (A\norm{\Gamma(u_t) - w^*_{\theta_t}} + A\norm{w^*_{\theta_t}} + B + 1).
\end{align}
For the second term,
we have
\begin{align}
  &\norm{F_{\theta_t}(\Gamma(u_t), Y_t) - \bar F_{\theta_t}(\Gamma(u_t))} \\
  \leq& \norm{F_{\theta_t}(\Gamma(u_t), Y_t)} + \norm{\bar F_{\theta_t}(\Gamma(u_t)) - \bar F_{\theta_t}(w^*_{\theta_t})} + \norm{w^*_{\theta_t}} \qq{(Assumption~\ref{assu uniform contraction}(i))}\\
  \leq&  U_F + L_F \norm{\Gamma(u_t)} + \norm{\bar F_{\theta_t}(\Gamma(u_t)) - \bar F_{\theta_t} (w^*_{\theta_t})} + \norm{w^*_{\theta_t}} \qq{(Lemma~\ref{lem bound of fxy})} \\
  =&  U_F + L_F \norm{\Gamma(u_t)} + \norm{\sum_y d_{\theta_t}(y) \left(F_{\theta_t}(\Gamma(u_t), y) - F_{\theta_t} (w^*_{\theta_t}, y) \right)} + \norm{w^*_{\theta_t}}  \\
  \leq& U_F + L_F \norm{\Gamma(u_t)} + L_F \norm{\Gamma(u_t) - w^*_{\theta_t}} + \norm{w^*_{\theta_t}} \\
  \leq& U_F + L_F \norm{\Gamma(u_t) - w^*_{\theta_t}} + L_F \norm{w^*_{\theta_t}} +  L_F \norm{\Gamma(u_t) - w^*_{\theta_t}} + \norm{w^*_{\theta_t}} \\
  \leq& A\norm{\Gamma(u_t) - w^*_{\theta_t}} + A \norm{w^*_{\theta_t}} + B.
\end{align}
Combining the two inequalities together yields
\begin{align}
  &\indot{\Gamma(u_t) - w^*_{\theta_t} - \left(\Gamma(u_{t-\tau_{\alpha_t}}) - w^*_{\theta_t}\right)}{F_{\theta_t}(\Gamma(u_t), Y_t) - \bar F_{\theta_t}(\Gamma(u_t))} \\
  \leq& 4 (L_wL_\theta + 1)  \alpha_{t-\tau_{\alpha_t}, t-1}(A\norm{\Gamma(u_t) - w^*_{\theta_t}} + C)^2 \\
  \leq& 8 (L_wL_\theta + 1)  \alpha_{t-\tau_{\alpha_t}, t-1}(A^2\norm{\Gamma(u_t) - w^*_{\theta_t}}^2 + C^2),
\end{align}
which completes the proof.
\end{proof}

\subsection{Proof of Lemma~\ref{lem bound t32}}
\label{sec proof lem bound t32}
\lemboundtthreetwo*
\begin{proof}
  \begin{align}
T_{32} = &\indot{\Gamma(u_{t-\tau_{\alpha_t}}) - w^*_{\theta_{t-\tau_{\alpha_t}}}}{F_{\theta_t}(\Gamma(u_t), Y_t) - F_{\theta_t}(\Gamma(u_{t- \tau_{\alpha_t}}), Y_t) + \bar F_{\theta_t}(\Gamma(u_{t- \tau_{\alpha_t}})) - \bar F_{\theta_t}(\Gamma(u_t))} \\
\leq & \norm{\Gamma(u_{t-\tau_{\alpha_t}}) - w^*_{\theta_{t-\tau_{\alpha_t}}}}\norm{F_{\theta_t}(\Gamma(u_t), Y_t) - F_{\theta_t}(\Gamma(u_{t- \tau_{\alpha_t}}), Y_t) + \bar F_{\theta_t}(\Gamma(u_{t- \tau_{\alpha_t}})) - \bar F_{\theta_t}(\Gamma(u_t))}.
  \end{align}
For the first term,
we have
\begin{align}
  \label{eq gradient bound dual norm}
  &\norm{\Gamma(u_{t-\tau_{\alpha_t}}) - w^*_{\theta_{t-\tau_{\alpha_t}}}} \\
  =&\norm{\Gamma(u_{t-\tau_{\alpha_t}}) - w^*_{\theta_{t-\tau_{\alpha_t}}} - (w^*_{\theta_t} - w^*_{\theta_t})} \\
  \leq& \norm{\Gamma(u_{t-\tau_{\alpha_t}}) - w^*_{\theta_t}} +  \norm{w^*_{\theta_t} - w^*_{\theta_{t-\tau_{\alpha_t}}}}  \\
  \leq& \norm{\Gamma(u_{t-\tau_{\alpha_t}}) - w^*_{\theta_t}} +  L_w L_\theta \alpha_{t-\tau_{\alpha_t}, t-1} \qq{(Lemma~\ref{lem accu learning rates})}\\
  \leq& \norm{\Gamma(u_{t-\tau_{\alpha_t}}) - \Gamma(u_t)} + \norm{\Gamma(u_t) - w^*_{\theta_t}} +  L_w L_\theta \alpha_{t-\tau_{\alpha_t}, t-1} \\
  \leq& \norm{\Gamma(u_t)} + \frac{B}{A} + \norm{\Gamma(u_t) - w^*_{\theta_t}} +  L_w L_\theta \alpha_{t-\tau_{\alpha_t}, t-1} \qq{(Lemma \ref{lem bound of xk diff})} \\
  \leq& (1 + L_wL_\theta \alpha_{t-\tau_{\alpha_t}, t-1}) \left(\norm{w^*_{\theta_t}} + \norm{\Gamma(u_t) - w^*_{\theta_t}} + \frac{B}{A} + \norm{\Gamma(u_t) - w^*_{\theta_t}} + 1\right) \\
  \leq& 2(1 + L_wL_\theta \alpha_{t-\tau_{\alpha_t}, t-1}) \left(U_w + \frac{B}{A} + \norm{\Gamma(u_t) - w^*_{\theta_t}} + 1\right) \\
  \leq& \frac{2(1 + L_wL_\theta \alpha_{t-\tau_{\alpha_t}, t-1})}{A} \left(A\norm{\Gamma(u_t) - w^*_{\theta_t}} + C\right).
\end{align}
For the second term,
\begin{align}
  &\norm{F_{\theta_t}(\Gamma(u_t), Y_t) - F_{\theta_t}(\Gamma(u_{t- \tau_{\alpha_t}}), Y_t) + \bar F_{\theta_t}(\Gamma(u_{t- \tau_{\alpha_t}})) - \bar F_{\theta_t}(\Gamma(u_t))} \\
  \leq &\norm{F_{\theta_t}(\Gamma(u_t), Y_t) - F_{\theta_t}(\Gamma(u_{t- \tau_{\alpha_t}}), Y_t)} + \norm{\bar F_{\theta_t}(\Gamma(u_{t- \tau_{\alpha_t}})) - \bar F_{\theta_t}(\Gamma(u_t))} \\
  \leq& L_F \norm{\Gamma(u_{t-\tau_{\alpha_t}}) - \Gamma(u_t)} + \norm{\sum_{y} d_{\theta_t}(y) \left(F_{\theta_t}(\Gamma(u_{t-\tau_{\alpha_t}}), y) - F_{\theta_t}(\Gamma(u_t), y)\right)} \\
  \leq& 2L_F \norm{\Gamma(u_{t-\tau_{\alpha_t}}) - \Gamma(u_t)} \\
  \leq& A \norm{\Gamma(u_{t-\tau_{\alpha_t}}) - \Gamma(u_t)} \\
  \leq& 4A \alpha_{t-\tau_{\alpha_t}, t-1} \left(A \norm{\Gamma(u_t)} + B\right) \qq{(Lemma~\ref{lem bound of xk diff})} \\
  \leq& 4A \alpha_{t-\tau_{\alpha_t}, t-1} (A\norm{\Gamma(u_t) - w^*_{\theta_t}} + A \norm{w^*_{\theta_t}} + B).
\end{align}
Combining the two inequalities together yields
\begin{align}
  T_{32} \leq& 8\alpha_{t-\tau_{\alpha_t}, t-1}(1 + L_wL_\theta \alpha_{t-\tau_{\alpha_t}, t-1})(A\norm{u_t - w^*_{\theta_t}} + C)^2 \\
  \leq& 16\alpha_{t-\tau_{\alpha_t}, t-1}(1 + L_wL_\theta \alpha_{t-\tau_{\alpha_t}, t-1})\left(A^2\norm{u_t - w^*_{\theta_t}}^2 + C^2\right),
\end{align}
which completes the proof.
\end{proof}

\subsection{Proof of Lemma~\ref{lem bound t331}}
\label{sec proof lem bound t331}
\lemboundtthreethreeone*
\begin{proof}
  \begin{align}
      \label{eq conditional independence t331}
      &\E\left[T_{331}\right] \\
      = &\E\left[\indot{\Gamma(u_{t-\tau_{\alpha_t}}) - w^*_{\theta_{t-\tau_{\alpha_t}}}}{F_{\theta_{t - \tau_{\alpha_t}}}(\Gamma(u_{t- \tau_{\alpha_t}}), \tilde Y_t) - \bar F_{\theta_{t-\tau_{\alpha_t}}}(\Gamma(u_{t- \tau_{\alpha_t}}))}\right] \\
      =&\E \left[ \E\left[\indot{\Gamma(u_{t-\tau_{\alpha_t}}) - w^*_{\theta_{t-\tau_{\alpha_t}}}}{F_{\theta_{t - \tau_{\alpha_t}}}(\Gamma(u_{t- \tau_{\alpha_t}}), \tilde Y_t) - \bar F_{\theta_{t-\tau_{\alpha_t}}}(\Gamma(u_{t- \tau_{\alpha_t}}))} \mid \substack{\theta_{t-\tau_{\alpha_t}} \\ u_{t-\tau_{\alpha_t}} \\ Y_{t-\tau_{\alpha_t}}} \right] \right]\\
      =&\E \left[ \indot{\Gamma(u_{t-\tau_{\alpha_t}}) - w^*_{\theta_{t-\tau_{\alpha_t}}}}{\E\left[F_{\theta_{t - \tau_{\alpha_t}}}(\Gamma(u_{t- \tau_{\alpha_t}}), \tilde Y_t) - \bar F_{\theta_{t-\tau_{\alpha_t}}}(\Gamma(u_{t- \tau_{\alpha_t}}))\mid \substack{\theta_{t-\tau_{\alpha_t}} \\ u_{t-\tau_{\alpha_t}} \\ Y_{t-\tau_{\alpha_t}}} \right] }\right] \\
      \leq&\E \left[ \norm{\Gamma(u_{t-\tau_{\alpha_t}}) - w^*_{\theta_{t-\tau_{\alpha_t}}}} \norm{\E\left[F_{\theta_{t - \tau_{\alpha_t}}}(\Gamma(u_{t- \tau_{\alpha_t}}), \tilde Y_t) - \bar F_{\theta_{t-\tau_{\alpha_t}}}(\Gamma(u_{t- \tau_{\alpha_t}}))\mid \substack{\theta_{t-\tau_{\alpha_t}} \\ u_{t-\tau_{\alpha_t}} \\ Y_{t-\tau_{\alpha_t}}} \right] } \right].
  \end{align}
We now bound the inner expectation.
\begin{align}
  &\norm{\E\left[F_{\theta_{t - \tau_{\alpha_t}}}(\Gamma(u_{t- \tau_{\alpha_t}}), \tilde Y_t) - \bar F_{\theta_{t-\tau_{\alpha_t}}}(\Gamma(u_{t- \tau_{\alpha_t}})) \mid \substack{\theta_{t-\tau_{\alpha_t}} \\u_{t-\tau_{\alpha_t}} \\ Y_{t-\tau_{\alpha_t}}}\right]} \\
  =&\norm{\sum_y \left(\Pr(\tilde Y_t = y \mid \substack{\theta_{t-\tau_{\alpha_t}} \\u_{t-\tau_{\alpha_t}} \\ Y_{t-\tau_{\alpha_t}}}) - d_{\theta_{t-\tau_{\alpha_t}}}(y) \right) F_{\theta_{t - \tau_{\alpha_t}}}(\Gamma(u_{t- \tau_{\alpha_t}}), y) } \\
  \leq &\max_y \norm{ F_{\theta_{t - \tau_{\alpha_t}}}(\Gamma(u_{t- \tau_{\alpha_t}}), y) } \sum_y \left|\Pr(\tilde Y_t = y \mid \substack{\theta_{t-\tau_{\alpha_t}} \\u_{t-\tau_{\alpha_t}} \\ Y_{t-\tau_{\alpha_t}}}) - d_{\theta_{t-\tau_{\alpha_t}}}(y) \right| \\
  \label{eq bound of xk-tk}
  \leq &\max_y \norm{ F_{\theta_{t - \tau_{\alpha_t}}}(\Gamma(u_{t- \tau_{\alpha_t}}), y) } \alpha_t \qq{(Definition of $\tau_{\alpha_t}$)} \\
  \leq & \alpha_t \left(U_F + L_F \norm{\Gamma(u_{t-\tau_{\alpha_t}})} \right) \qq{(Lemma \ref{lem bound of fxy})} \\
  \leq & \alpha_t \left(U_F + L_F \norm{\Gamma(u_{t-\tau_{\alpha_t}}) - \Gamma(u_t)} + L_F \norm{\Gamma(u_t)}\right) \\
  \leq & \alpha_t \left(B + A \left(\norm{\Gamma(u_t)} + \frac{B}{A}\right) + A \norm{\Gamma(u_t)}\right) \qq{(Lemma \ref{lem bound of xk diff})} \\
  \leq & \alpha_t \left(2B + (A + 1) \norm{\Gamma(u_t)}\right) \\
  \leq & 2\alpha_t \left(B + A \norm{\Gamma(u_t)}\right) \\
  \leq & 2\alpha_t \left(B + A \norm{\Gamma(u_t) - w^*_{\theta_t}} + A \norm{w^*_{\theta_t}}\right) \\
  \leq & 2\alpha_t \left(A \norm{\Gamma(u_t) - w^*_{\theta_t}} + C\right) \\
\end{align}
Using the above inequality and \eqref{eq gradient bound dual norm} yields
\begin{align}
  &\E\left[T_{331}\right] \\
  \leq & \E \left[\frac{4 \alpha_t(1 + L_wL_\theta \alpha_{t-\tau_{\alpha_t}, t-1})}{A}\left(A\norm{\Gamma(u_t) - w^*_{\theta_t}} + C\right)^2\right] \\
  \leq & \E \left[\frac{8 \alpha_t(1 + L_wL_\theta \alpha_{t-\tau_{\alpha_t}, t-1})}{A}\left(A^2\norm{\Gamma(u_t) - w^*_{\theta_t}}^2 + C^2\right)\right],
\end{align}
which completes the proof.
\end{proof}

\subsection{Proof of Lemma~\ref{lem bound t332}}
\label{sec proof lem bound t332}
\lemboundtthreethreetwo*
\begin{proof}
  \begin{align}
      &\E\left[T_{332}\right] \\
      = &\E\left[\indot{\Gamma(u_{t-\tau_{\alpha_t}}) - w^*_{\theta_{t-\tau_{\alpha_t}}}}{F_{\theta_{t - \tau_{\alpha_t}}}(\Gamma(u_{t- \tau_{\alpha_t}}), Y_t) -F_{\theta_{t - \tau_{\alpha_t}}}(\Gamma(u_{t- \tau_{\alpha_t}}), \tilde Y_t)}\right] \\
      \leq &  \E\left[\norm{\Gamma(u_{t-\tau_{\alpha_t}}) - w^*_{\theta_{t-\tau_{\alpha_t}}}} \norm{ \E \left[{F_{\theta_{t - \tau_{\alpha_t}}}(\Gamma(u_{t- \tau_{\alpha_t}}), Y_t) -F_{\theta_{t - \tau_{\alpha_t}}}(\Gamma(u_{t- \tau_{\alpha_t}}), \tilde Y_t)} \mid \substack{u_{t-\tau_{\alpha_t}} \\ \theta_{t-\tau_{\alpha_t}} \\ Y_{t-\tau_{\alpha_t}}} \right]}\right] \\
      \intertext{\hfill (Similar to \eqref{eq conditional independence t331})}
      \leq& \E \Bigg[ \frac{2(1 + L_wL_\theta \alpha_{t-\tau_{\alpha_t}, t-1})}{A} \left(A\norm{\Gamma(u_t) - w^*_{\theta_t}} + C\right)\\
      &\times 2\ny L_PL_\theta \sum_{j=t-\tau_{\alpha_t}}^{t-1}\alpha_{t-\tau_{\alpha_t}, j}\left(A \norm{\Gamma(u_t) - w^*_{\theta_t}} + C\right) \Bigg] \\
      \intertext{\hfill (Using \eqref{eq gradient bound dual norm} and Lemma \ref{lem chain difference bound})}
      \leq & \frac{8 \ny L_P L_\theta \sum_{j=t-\tau_{\alpha_t}}^{t-1}\alpha_{t-\tau_{\alpha_t}, j} (1 + L_wL_\theta \alpha_{t-\tau_{\alpha_t}, t-1})}{A} \left(A^2 \E\left[\norm{\Gamma(u_t) - w^*_{\theta_t}}^2\right] + C^2\right),
  \end{align}
  which completes the proof.
\end{proof}

\subsection{Proof of Lemma~\ref{lem bound of t333}}
\label{sec proof lem bound t333}
\lemboundoftthreethreethree*
\begin{proof}
  \begin{align}
      T_{333} = &\indot{\Gamma(u_{t-\tau_{\alpha_t}}) - w^*_{\theta_{t-\tau_{\alpha_t}}}}{F_{\theta_{t}}(\Gamma(u_{t- \tau_{\alpha_t}}), Y_t) -F_{\theta_{t - \tau_{\alpha_t}}}(\Gamma(u_{t- \tau_{\alpha_t}}), Y_t)}  \\
      \leq &\norm{\Gamma(u_{t-\tau_{\alpha_t}}) - w^*_{\theta_{t-\tau_{\alpha_t}}}} \norm{F_{\theta_{t}}(\Gamma(u_{t- \tau_{\alpha_t}}), Y_t) -F_{\theta_{t - \tau_{\alpha_t}}}(\Gamma(u_{t- \tau_{\alpha_t}}), Y_t)}  \\
      \leq & \frac{2(1 + L_wL_\theta \alpha_{t-\tau_{\alpha_t}, t-1})}{A}\left(A\norm{\Gamma(u_t) - w^*_{\theta_t}} + C\right) \\
      &\times L_F' L_\theta \alpha_{t-\tau_{\alpha_t}, t-1} \left(\norm{\Gamma(u_{t-\tau_{\alpha_t}})} + U_F' \right) \qq{(Using \eqref{eq gradient bound dual norm} and Lemma~\ref{lem accu learning rates})}.
  \end{align}
Since 
\begin{align}
  \label{eq tmp 3}
  &\norm{\Gamma(u_{t-\tau_{\alpha_t}})} \\
  \leq & \norm{\Gamma(u_{t-\tau_{\alpha_t}}) - \Gamma(u_t)} + \norm{\Gamma(u_t)} \\
  \leq & 2\norm{\Gamma(u_t)} + \frac{B}{A} \qq{(Lemma \ref{lem bound of xk diff})} \\
  \leq & 2\norm{\Gamma(u_t) - w^*_{\theta_t}} + 2\norm{w^*_{\theta_t}} + \frac{B}{A},
\end{align}
we have
\begin{align}
  T_{333} \leq \frac{8L_F' L_\theta \alpha_{t-\tau_{\alpha_t}, t-1} (1 + L_wL_\theta \alpha_{t-\tau_{\alpha_t}, t-1})}{A^2 }\left(A^2 \norm{\Gamma(u_t) - w^*_{\theta_t}}^2 + C^2\right),
\end{align}
which completes the proof.
\end{proof}

\subsection{Proof of Lemma~\ref{lem bound t334}}
\label{sec proof lem bound t334}
\lemboundtthreethreefour*
\begin{proof}
  \begin{align}
      T_{334} = &\indot{\Gamma(u_{t-\tau_{\alpha_t}}) - w^*_{\theta_{t-\tau_{\alpha_t}}}}{\bar F_{\theta_{t - \tau_{\alpha_t}}}(\Gamma(u_{t- \tau_{\alpha_t}})) - \bar F_{\theta_{t}}(\Gamma(u_{t- \tau_{\alpha_t}}))}  \\
      \leq &\norm{\Gamma(u_{t-\tau_{\alpha_t}}) - w^*_{\theta_{t-\tau_{\alpha_t}}}} \norm{\bar F_{\theta_{t}}(\Gamma(u_{t- \tau_{\alpha_t}})) - \bar F_{\theta_{t - \tau_{\alpha_t}}}(\Gamma(u_{t- \tau_{\alpha_t}}))}  \\
      \leq & \frac{2(1 + L_wL_\theta \alpha_{t-\tau_{\alpha_t}, t-1})}{A}\left(A\norm{\Gamma(u_t) - w^*_{\theta_t}} + C\right) \\
      &\times L_F'' L_\theta \alpha_{t-\tau_{\alpha_t}, t-1} \left(\norm{\Gamma(u_{t-\tau_{\alpha_t}})} + U_F''\right) \qq{(Using \eqref{eq gradient bound dual norm} and Lemma~\ref{lem accu learning rates})}.
  \end{align}
  Using \eqref{eq tmp 3} completes the proof.
\end{proof}

\subsection{Proof of Lemma~\ref{lem bound t5}}
\label{sec proof lem bound t5}
\lemboundtfive*
\begin{proof}
  \begin{align}
      T_5 =&\norm{F_{\theta_t}(\Gamma(u_t), Y_t) - \Gamma(u_t)}^2 \\
      \leq& \left(\norm{F_{\theta_t}(\Gamma(u_t), Y_t)} + \norm{\Gamma(u_t)} \right)^2  \\
      \leq& \left(U_F + (L_F + 1) \norm{\Gamma(u_t)}\right)^2 \qq{(Lemma~\ref{lem bound of fxy})} \\
      \leq& \left(B + A \norm{\Gamma(u_t)}\right)^2  \\
      \leq& \left(B + A \norm{\Gamma(u_t) - w^*_{\theta_t}} + A\norm{w^*_{\theta_t}}\right)^2  \\
      \leq& 2\left(A^2 \norm{\Gamma(u_t) - w^*_{\theta_t}}^2 + C^2\right)
  \end{align}
\end{proof}

\begin{lemma}
  \label{lem bound of fxy}
  For any $\theta, w, y$,
  \begin{align}
      \norm{F_{\theta}(w, y)} \leq U_F + L_F \norm{w}
  \end{align}
\end{lemma}
\begin{proof}
  Assumption \ref{assu regularization} implies that
  \begin{align}
      \norm{F_{\theta}(w, y)} - \norm{F_{\theta}(0, y)}  &\leq \norm{F_{\theta}(0, y) - F_{\theta}(w, y)} \\
      &\leq L_F \norm{w - 0},
  \end{align}
  which completes the proof.
\end{proof}

\begin{lemma}
  \label{lem chain difference bound}
  \begin{align}
      &\norm{ \E \left[{F_{\theta_{t - \tau_{\alpha_t}}}(\Gamma(u_{t - \tau_{\alpha_t}}), Y_t) -F_{\theta_{t - \tau_{\alpha_t}}}(\Gamma(u_{t - \tau_{\alpha_t}}), \tilde Y_t)} \mid \substack{u_{t-\tau_{\alpha_t}} \\ \theta_{t-\tau_{\alpha_t}} \\ Y_{t-\tau_{\alpha_t}}} \right]} \\
      \leq &2\ny L_PL_\theta \sum_{j=t-\tau_{\alpha_t}}^{t-1}\alpha_{t-\tau_{\alpha_t}, j}(A \norm{\Gamma(u_t) - w^*_{\theta_t}} + C)
  \end{align}
\end{lemma}
\begin{proof}
  In this proof, all $\Pr$ and $\E$ are implicitly conditioned on $u_{t-\tau_{\alpha_t}}, \theta_{t-\tau_{\alpha_t}}, Y_{t-\tau_{\alpha_t}}$. 
  We use $\Theta_t$ to denote the set of all possible $\theta_t$ given $u_{t-\tau_{\alpha_t}}, \theta_{t-\tau_{\alpha_t}}, Y_{t-\tau_{\alpha_t}}$.
  Obviously, $\Theta_t$ is a finite set.
  We have
  \begin{align}
      &\Pr(Y_t = y') \\
      =& \sum_{y} \sum_{z \in \Theta_{t}} \Pr(Y_t = y' , Y_{t-1} = y, \theta_{t} = z) \\
      =& \sum_{y} \sum_{z \in \Theta_{t}} \Pr(Y_t = y' \mid Y_{t-1} = y, \theta_{t} = z) \Pr(Y_{t-1} = y, \theta_{t} = z) \\
      =& \sum_{y} \sum_{z \in \Theta_{t}} P_{z}(y, y') \Pr(Y_{t-1} = y) \Pr(\theta_{t} = z | Y_{t-1} = y)
  \end{align}
  \begin{align}
      &\Pr(\tilde Y_t = y') \\
      =& \sum_{y} \Pr(\tilde Y_{t-1} = y) P_{\theta_{t-\tau_{\alpha_t}}}(y, y') \\
      =& \sum_{y} \Pr(\tilde Y_{t-1} = y) P_{\theta_{t-\tau_{\alpha_t}}}(y, y') \sum_{z \in \Theta_{t}} \Pr(\theta_{t} = z | Y_{t-1} = y) \\
      =& \sum_{y} \sum_{z \in \Theta_{t}} \Pr(\tilde Y_{t-1} = y) P_{\theta_{t-\tau_{\alpha_t}}}(y, y')  \Pr(\theta_{t} = z | Y_{t-1} = y)
  \end{align}
  Consequently,
  \begin{align}
      &\sum_{y'} \left|\Pr(Y_t = y') - \Pr(\tilde Y_t = y') \right| \\
      \leq &\sum_{y, y'} \sum_{z \in \Theta_{t}} \left| \Pr(Y_{t-1} = y) P_z(y, y') - \Pr(\tilde Y_{t-1} = y) P_{\theta_{t-\tau_{\alpha_t}}}(y, y')\right| \Pr(\theta_{t} = z \mid Y_{t-1} = y).
  \end{align}
Since for any $z \in \Theta_{t}$,
\begin{align}
  &\left| \Pr(Y_{t-1} = y) P_z(y, y') - \Pr(\tilde Y_{t-1} = y) P_{\theta_{t-\tau_{\alpha_t}}}(y, y')\right| \\
  \leq & \left| \Pr(Y_{t-1} = y) P_z(y, y') - \Pr(\tilde Y_{t-1} = y) P_{z}(y, y')\right| \\
  &+ \left| \Pr(\tilde Y_{t-1} = y) P_z(y, y') - \Pr(\tilde Y_{t-1} = y) P_{\theta_{t-\tau_{\alpha_t}}}(y, y')\right| \\
  \leq & \left| \Pr(Y_{t-1} = y) - \Pr(\tilde Y_{t-1} = y) \right| P_z(y, y') + L_PL_\theta \alpha_{t-\tau_{\alpha_t}, t-1} \Pr(\tilde Y_{t-1} = y) \qq{(Lemma~\ref{lem accu learning rates})},
\end{align}
we have
\begin{align}
  &\sum_{y'} \left|\Pr(Y_t = y') - \Pr(\tilde Y_t = y') \right| \\
  \leq & \sum_y \left| \Pr(Y_{t-1} = y) - \Pr(\tilde Y_{t-1} = y) \right| + \ny L_PL_\theta \alpha_{t-\tau_{\alpha_t}, t-1}.
\end{align}
Applying the above inequality recursively yields
\begin{align}
  \label{eq y difference two chains}
  \sum_{y'} \left|\Pr(Y_t = y') - \Pr(\tilde Y_t = y') \right| \leq \ny L_PL_\theta \sum_{j=t-\tau_{\alpha_t}}^{t-1} \alpha_{t-\tau_{\alpha_t}, j}.
\end{align}
Consequently,
\begin{align}
  &\norm{ \E \left[{F_{\theta_{t - \tau_{\alpha_t}}}(\Gamma(u_{t - \tau_{\alpha_t}}), Y_t) -F_{\theta_{t - \tau_{\alpha_t}}}(\Gamma(u_{t - \tau_{\alpha_t}}), \tilde Y_t)} \right]} \\
  =& \norm{\sum_y \left(\Pr(Y_t = y) - \Pr(\tilde Y_t = y)\right) F_{\theta_{t-\tau_{\alpha_t}}}(\Gamma(u_{t-\tau_{\alpha_t}}), y)} \\
  \leq &\max_y \norm{F_{\theta_{t-\tau_{\alpha_t}}}(\Gamma(u_{t-\tau_{\alpha_t}}), y)} \ny L_PL_\theta \sum_{j=t-\tau_{\alpha_t}}^{t-1}\alpha_{t-\tau_{\alpha_t}, j} \\
  \leq & 2\ny L_PL_\theta \sum_{j=t-\tau_{\alpha_t}}^{t-1}\alpha_{t-\tau_{\alpha_t}, j}(A \norm{\Gamma(u_t) - w^*_{\theta_t}} + C) \qq{(Using \eqref{eq bound of xk-tk}),}
\end{align}
which completes the proof.
\end{proof}

\end{document}